%% file: 00_main.tex
\documentclass[sigconf]{acmart}

\usepackage{amsthm}
\usepackage{subcaption}

\theoremstyle{definition}
\newtheorem{definition}{Definition}
\newtheorem{thm}{Theorem}
\newtheorem{remark}{Remark}

\newcommand{\casper}{{\textsc{Casper}}}

\AtBeginDocument{%
  }

\copyrightyear{2024}
\acmYear{2024}
\setcopyright{acmlicensed}\acmConference[CIKM'24]{Proceedings of the 33rd ACM International Conference on Information and Knowledge Management}{October 21--25, 2024}{Boise, ID, USA}
\acmBooktitle{Proceedings of the 33rd ACM International Conference on Information and Knowledge Management (CIKM '24), October 21--25, 2024, Boise, ID, USA}
\acmDOI{10.1145/3627673.3679642}
\acmISBN{979-8-4007-0436-9/24/10}

\makeatletter
\gdef\@copyrightpermission{
  \begin{minipage}{0.3\columnwidth}\href{https://creativecommons.org/licenses/by/4.0/}{\includegraphics[width=0.90\textwidth]{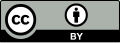}} 
  \end{minipage}\hfill
  \begin{minipage}{0.7\columnwidth}\href{https://creativecommons.org/licenses/by/4.0/}{This work is licensed under a Creative Commons Attribution International 4.0 License.} 
  \end{minipage}
  \vspace{5pt}
}
\makeatother

\begin{document}

\title{Causality-Aware Spatiotemporal Graph Neural Networks for Spatiotemporal Time Series Imputation}

\author{Baoyu Jing}
\orcid{0000-0003-1564-6499} 
\affiliation{
  \institution{University of Illinois}
  \city{Urbana-Champaign}
  \state{IL}
  \country{USA}
}
\email{baoyuj2@illinois.edu}

\author{Dawei Zhou}
\orcid{0000-0002-7065-2990} 
\affiliation{
  \institution{Virginia Polytechnic Institute and State University}
  \city{Blacksburg}
  \state{VA}
  \country{USA}
}
\email{zhoud@vt.edu}

\author{Kan Ren}
\orcid{0000-0002-4032-9615} 
\affiliation{
  \institution{Microsoft Research Asia}
  \city{Shanghai}
  \country{China}
}
\email{kan.ren@microsoft.com}

\author{Carl Yang}
\orcid{0000-0001-9145-4531} 
\affiliation{
  \institution{Emory University}
  \city{Atlanta}
  \state{GA}
  \country{USA}
}
\email{j.carlyang@emory.edu}

\renewcommand{\shortauthors}{Baoyu Jing, Dawei Zhou, Kan Ren, \& Carl Yang}

\begin{abstract}
% Spatiotemporal time series is the foundation of understanding human activities and their impacts, which is usually collected via monitoring sensors placed at different locations.
Spatiotemporal time series are usually collected via monitoring sensors placed at different locations, which usually contain missing values due to various failures, such as mechanical damages and Internet outages.
Imputing the missing values is crucial for analyzing time series.
% To impute the missing values, a lot of methods have been proposed.
When recovering a specific data point, most existing methods consider all the information relevant to that point regardless of the cause-and-effect relationship. 
During data collection, it is inevitable that some unknown confounders are included, e.g., background noise in time series and non-causal shortcut edges in the constructed sensor network.
These confounders could open backdoor paths and establish non-causal correlations between the input and output.
Over-exploiting these non-causal correlations could cause overfitting.
In this paper, we first revisit spatiotemporal time series imputation from a causal perspective and show how to block the confounders via the frontdoor adjustment.
% which shows the causal relationships among the input, output, embeddings, and confounders.
% Next, we show how to block the confounders via the frontdoor adjustment.
Based on the results of frontdoor adjustment, we introduce a novel \underline{C}ausality-\underline{A}ware \underline{Sp}atiot\underline{e}mpo\underline{r}al Graph Neural Network (\casper), which contains a novel Prompt Based Decoder (PBD) and a Spatiotemporal Causal Attention (SCA).
PBD could reduce the impact of confounders and SCA could discover the sparse causal relationships among embeddings.
Theoretical analysis reveals that SCA discovers causal relationships based on the values of gradients.
We evaluate \casper\ on three real-world datasets, and the experimental results show that \casper\ could outperform the baselines and could effectively discover the causal relationships.

\end{abstract}

\begin{CCSXML}
<ccs2012>
<concept>
<concept_id>10002951.10003227.10003351</concept_id>
<concept_desc>Information systems~Data mining</concept_desc>
<concept_significance>500</concept_significance>
</concept>
<concept>
<concept_id>10002951.10003227.10003236.10003238</concept_id>
<concept_desc>Information systems~Sensor networks</concept_desc>
<concept_significance>500</concept_significance>
</concept>
</ccs2012>
\end{CCSXML}

\ccsdesc[500]{Information systems~Data mining}
\ccsdesc[500]{Information systems~Sensor networks}

\keywords{Spatiotemporal Time Series Imputation, Spatiotemporal Graph Neural Network, Causal Attention}

\maketitle

\input{01_introduction}
\input{02_preliminary}
\input{03_method}
\input{04_experiments}
\input{05_related_work}
\input{06_conclusion}

\bibliographystyle{ACM-Reference-Format}
\bibliography{sample-base}

\end{document}

%% file: 01_introduction.tex
\section{Introduction}
Spatiotemporal data mining \cite{atluri2018spatio} is the cornerstone of analyzing and understanding the patterns of spacetime and human activities, such as environmental monitoring \cite{li2018diffusion,roach2020canon,zheng2015forecasting,marisca2022learning}, e-business \cite{kang2018self,zhou2020data,jing2024sterling,jing2022coin,yan2024reconciling,li2022graph,jing2021hdmi,yan2022dissecting} and social science \cite{zhao2017modeling,du2021new,yan2021bright,yan2024pacer,zeng2023parrot,zeng2024hierarchical,feng2022adversarial,feng2024ariel,zhou2023closed}.
% e.g., traffic flows \cite{li2018diffusion,roach2020canon} and greenhouse gas emissions \cite{zheng2015forecasting}.
% which is the basis for formulating public policies, e.g., the traffic transportation policy
% \footnote{\url{https://www.transportation.gov/policy/transportation-policy}} 
% and the Clean Air Act.
% \footnote{\url{https://www.epa.gov/laws-regulations/summary-clean-air-act}}.
Time series \cite{jing2024automated,wang2023networked,li2021outlier,zheng2024mulan} is one of the most common data types, which is usually collected by monitoring sensors.
For example, traffic flow time series \cite{li2018diffusion}, e.g., speed, is recorded by the radar sensors on roads.
Air pollution time series \cite{zheng2015forecasting}, e.g., concentrations of PM2.5, is collected from air quality monitoring sites across cities.

In the real world, it is not uncommon that the collected spatiotemporal time series is incomplete with missing data due to various failures, 
% e.g., the Internet is disconnected, the battery of the sensors is exhausted, and the sensors have mechanical damages.
e.g., sensors have mechanical damage.
The missing data usually significantly impacts the process and conclusion of data analysis.
Therefore, how to reconstruct the missing data from the observed data, i.e., imputation, is a fundamental problem of spatiotemporal time series analysis.
% Traditional methods leverages simple interpolations, e.g., K-Nearest Neighbors (KNN) and matrix factorization, to recover the missing values.
% Theses methods are incapable of modeling non-linear relationships among data.
% Equipped with non-linear activation functions, deep learning methods become the mainstream for time series imputation.
In recent years, deep learning methods become the mainstream for time series imputation.
Most existing deep time series imputation methods \cite{cao2018brits,luo2018multivariate,luo2019e2gan} use Recurrent Neural Network (RNN) to capture the temporal dynamics of time series and autoregressively recover the missing data by the predicted values.
Recent deep learning methods~\cite{yildiz2022multivariate} propose to use non-autoregressive structures, e.g., Transformer \cite{vaswani2017attention}, to avoid the progressive error propagation incurred via the autoregression in RNN by concurrently considering the entire input context.
However, these methods only consider the temporal patterns yet overlook the spatial relationships among sensors, e.g., geographical distances.
To further account for spatial relationships, graph neural networks \cite{kipf2016semi,velivckovic2017graph,yan2024trainable,zenggraph,xuslog,zheng2021deeper} are extended to the spatiotemporal setting \cite{jing2021network,huang2020learning,marisca2022learning,wang2023networked}.
Although these methods have achieved impressive performance in recovering the missing values, 
% when recovering a data point, 
they tend to include all the available information related to the missing point as references without distinguishing whether there is a causal relationship between them.
% they mainly focus on minimizing the reconstruction losses, e.g., Mean Absolute Error (MAE), yet largely overlook the causal and non-causal relationships of time series.

When collecting datasets, it is inevitable to include some unknown confounders \cite{pearl2018book}.
For example, the background noise might be recorded, and non-causal shortcut edges might be built for two sensors.
Let's take the air monitoring sensor network as a concrete example to understand the non-causal edges.
A common practice to build the network is adding an edge for two sensors if their distance is below a threshold \cite{marisca2022learning}. 
Although simple and usually effective, the distance-based network does not necessarily imply the real causality between sensors.
In the real world, air flow between two locations could be influenced by other factors, e.g., wind direction and terrain. 
An example is shown in Figure \ref{fig:maps} in Section \ref{sec:exp_vis}.
Simply exploiting the shortcut edges without discovering the causality could make the model overfit the training data and be vulnerable to noise during inference.  
% As shown in Figure \ref{fig:scm}, these confounders could open shortcut backdoor paths between the input data and output predictions \cite{pearl2018book}.
% Recent studies reveal that models might take advantage of the backdoor paths to make decisions instead of struggling to discover the real cause-and-effect relationships  \cite{geirhos2020shortcut, sui2022causal,yang2021causal}.
% Over-reliance on the non-causal correlations could lead to overfitting and make the model vulnerable to noise.
% For example, when recovering the missing point for the sensor $i$, the model might exploit its neighbor sensors $\mathcal{N}(i)$ in the sensor network without distinguishing whether the information of the neighbors could contribute to the recovery or not.
% When the neighbors $\mathcal{N}(i)$ have similar patterns as the sensor $i$, such a practice could effectively recover the missing point.
% However, when a neighbor $i'\in\mathcal{N}(i)$ contains a lot of noise or does not share similar patterns as the sensor $i$, its information could mislead the model.

To reduce the negative effects brought by confounders, we first review the process of spatiotemporal time series imputation from a causal perspective \cite{pearl2018book} to show the causal relationships among the input, output, embeddings, and confounders.
The results show that confounders could establish undesired non-causal shortcut backdoor paths between the input and output.
Then, we show how to eliminate the backdoor paths via the frontdoor adjustment \cite{pearl2018book}.
Based on the results of the frontdoor adjustment, we introduce a novel \underline{C}ausality-\underline{A}ware \underline{Sp}atiot\underline{e}mpo\underline{r}al Graph Neural Network (\casper), which is equipped with a novel Prompt Based Decoder (PBD) and a Spatiotemporal Causal Attention (SCA). 
The proposed PBD effectively reduces the impact of unknown confounders by injecting the global context information of datasets into the embeddings.
Inspired by \cite{jia2022visual}, which uses learnable prompts to capture the contextual information of downstream tasks when tuning visual models, 
PBD leverages prompts to learn the contextual information of datasets automatically rather than employing external models to approximate the context.
To further enforce sparse causality between embeddings, we introduce SCA, which determines the cause-and-effect relationship via a causal gate.
It can be theoretically proven that the proposed causal gate (1) enforces the sparsity since it converges to 0 or 1; 
(2) is a gradient-based explanation similar to \cite{selvaraju2017grad}, which determines the causality based on the values of gradients.
We extensively evaluate \casper\ on three real-world datasets.
The experimental results show that \casper\ could significantly outperform baselines and could effectively discover causality.

The major contributions of the paper are summarized as follows:
\begin{itemize}
    \item We review the spatiotemporal time series imputation task from a causal perspective, where we show the problems of the undesired confounders.
    Then, we show how to eliminate confounders via the frontdoor adjustment.
    \item We propose a novel Causality-Aware Spatiotemporal Graph Neural Network (\casper) based on the frontdoor adjustment.
    \casper\ is equipped with a novel Spatiotemporal Causal Attention (SCA) and a Prompt Based Decoder (PBD).
    PBD effectively blocks the backdoor paths and SCA explicitly reveals the causality between embeddings.
    \item We provide theoretical analysis to deeply understand how \casper\ determines causal and non-causal relationships.
    \item We evaluate \casper\ on three real-world datasets. The experimental results show that \casper\ could outperform the baselines and effectively discover causal relationships.
\end{itemize}

%% file: 02_preliminary.tex
\section{Preliminary}
% \subsection{Spatiotemporal Time Series Imputation}
In this section, we briefly introduce the definitions of spatiotemporal time series and spatiotemporal time series imputation.
We also review the definitions of Granger causality and attention function.

\begin{definition}[Incomplete Spatiotemporal Time Series]
% We denote a complete spatiotemporal time series as: $\Tilde{\mathbf{G}}=(\mathbf{X}, \mathbf{A})$, where $\mathbf{X}\in\mathbb{R}^{N\times T}$ is the multivariate time series collected from $N$ sensors with totally $T$ steps,
% $\mathbf{A}\in\mathbb{R}^{N\times N}$ is the adjacency matrix of the sensor network.
We denote an incomplete spatiotemporal time series with missing values as $\mathbf{G}=(\mathbf{X}, \mathbf{A}, \mathbf{M})$, where $\mathbf{X}\in\mathbb{R}^{N\times T}$ is the multivariate time series collected from $N$ sensors with totally $T$ steps,
$\mathbf{A}\in\mathbb{R}^{N\times N}$ is the adjacency matrix of the sensor network,
$\mathbf{M}\in\{0, 1\}^{N\times T}$ is the binary mask and 0/1 denotes the absence/presence of a data point.
% If $\mathbf{X}$ contains missing data, then we denote such a spatiotemporal time series as $\mathbf{G}=(\mathbf{X}, \mathbf{A}, \mathbf{M})$, where $\mathbf{M}\in\{0, 1\}^{N\times T}$ is the binary mask and 0/1 denotes the absence/presence of a data point.
    
\end{definition}

\begin{definition}[Spatiotemporal Time Series Imputation]
Given an incomplete spatiotemporal time series $\mathbf{G}=(\mathbf{X}, \mathbf{A}, \mathbf{M})$,
we denote
% $\Tilde{\mathbf{X}}\in\mathbf{R}^{N\times T\times c}$ 
$\mathbf{Y}\in\mathbf{R}^{N\times T}$ 
as the complete time series of $\mathbf{X}$,
% $\mathbf{X}=\mathbf{M}\odot\Tilde{\mathbf{X}}$, 
such that $\mathbf{X}=\mathbf{M}\odot{\mathbf{Y}}$, 
where $\odot$ is the Hadamard product.
The task is to build a function $\hat{\mathbf{Y}}=f(\mathbf{G})$ to minimize the reconstruction error, e.g., Mean Absolute Error (MAE), between $\hat{\mathbf{Y}}$ and 
% $\Tilde{\mathbf{X}}$.
$\mathbf{Y}$.
% Note that minimizing MAE is equivalent to maximizing the log-likelihood of Laplace distribution \cite{hodson2022root}. 
\end{definition}

% \subsection{Granger Causality}

\begin{definition}[Granger Causality \cite{granger1969investigating,cheng2022cuts}]\label{def:granger_causality}
Let $\mathbf{X}\in\mathbb{R}^{N\times T}$ be the values of past $T$ steps of $N$ time series, and $\hat{y}_{i,T+1}=f_i(\mathbf{X})\in\mathbb{R}$ be the \emph{time series forecasting function} predicting the future value of the $i$-th time series at the $T+1$ step, where $i\in\{1,\dots, N\}$.
% $\mathbf{X}_i$
The $i'$-th time series is said to
Granger cause
% $\mathbf{X}_j$ 
the $i$-th time series
if there exists a point $x'_{i',t'}\neq x_{i',t'}$, $t'\in\{1,\dots,T\}$, such that $f_i(\mathbf{X}')\neq f_i(\mathbf{X})$, where $\mathbf{X}'$ is obtained by replacing $x_{i',t'}$ in $\mathbf{X}$ with $x'_{i',t'}$.
\end{definition}

Generally, if $x_{i',t'}$ impacts the prediction of the future value of the $i$-th time series, then the $i'$-th time series Granger causes the $i$-th time series.
In the case that $f_i$ is a linear model:
\begin{equation}\label{eq:linear}
    \hat{y}_{i,T+1}=f_i(\mathbf{X})=\sum_{i'=1,t'=1}^{N,T}\alpha_{i',t'}x_{i',t'}.
\end{equation}
if the coefficient $\alpha_{i',t'}\neq0$, then the $i'$-th time series Granger causes the $i$-th time series.

% \begin{definition}[Attention Function]\label{def:attention}
% Let $\mathbf{q}$, $\{\mathbf{k}_{i}\}_{i=1}^{N}$ and $\{\mathbf{v}_{i}\}_{i=1}^{N}$ be the $d$-dimensional query, key and value embeddings, and let $s$ be the scoring function, then the attention function is defined as:
% \begin{equation}\label{eq:attetnion}
%     \mathbf{h} = \sum_{i=1}^N\alpha_i\mathbf{v}_i,\, \alpha_i=\frac{\exp(s(\mathbf{q},\mathbf{k}_{i}))}{\sum_{i'=1}^N\exp(s(\mathbf{q},\mathbf{k}_{i'}))},\, \mathbf{v}_i=f_v(\mathbf{q},\mathbf{k}_i)
% \end{equation}
% where $\mathbf{h}\in\mathbb{R}^d$ is the output embedding, and $f_v$ is a neural network module.
% \end{definition}

\begin{definition}[Attention Function]\label{def:attention}
Let $\mathbf{q}$ and $\{\mathbf{k}_{i}\}_{i=1}^{N}$ be the $d$-dimensional query and key embeddings, 
let $s$ and $f_v$ be the scoring and message functions, then the attention function is defined as:
\begin{equation}\label{eq:attetnion}
    \mathbf{h} = \sum_{i=1}^N\alpha_i\mathbf{v}_i,\quad \alpha_i=\frac{\exp(s(\mathbf{q},\mathbf{k}_{i}))}{\sum_{i'=1}^N\exp(s(\mathbf{q},\mathbf{k}_{i'}))},\quad \mathbf{v}_i=f_v(\mathbf{q},\mathbf{k}_i),
\end{equation}
where $\mathbf{h}\in\mathbb{R}^d$ is the output, and $\mathbf{v}_i$ is the message from $\mathbf{k}_i$ to $\mathbf{q}$.
\end{definition}

% Attention functions are critical embedding updating functions in deep learning models and the attention scores could indicate the correlation between emebeddings.
% However, correlation does not necessarily imply causality.
% In this paper, we introduce a novel Spatiotemporal Causal Attention (SCA) to capture the causal relationship between embeddings.

%% file: 03_method.tex
\section{Methodology}
In this section, we first provide a causal view of the spatiotemporal imputation task, and show how to eliminate the impact of unknown confounds by frontdoor adjustment, based on which, we introduce a novel Causality-Aware Spatiotemporal Graph Neural Network (\casper).
Finally, we provide further analysis of \casper.

\subsection{Causal View of Spatiotemporal Imputation}
% \subsection{Probabilistic Causation for Spatiotemporal Imputation}
% In this subsection, we first formally extend the definition of Granger causality to imputation and then introduce the causal 

% \textbf{Granger Causality for Imputation.} 
% Granger causality \cite{granger1969investigating} (Definition \ref{def:granger_causality}) is originally defined for time series forecasting, in this paper, we extend Granger causality to imputation in Definition~\ref{def:granger_causality_imputation}:

% \begin{definition}[Granger Causality for Imputation]\label{def:granger_causality_imputation}
% Let $\mathbf{X}\in\mathbb{R}^{N\times T}$ be $N$ incomplete time series, 
% $\mathbf{M}\in[0,1]^{N\times T}$ be the corresponding mask,
% and $\hat{y}_{i,t}=f_{i,t}(\mathbf{X}, \mathbf{M})$ be the \emph{time series imputation function}, where $i\in[1,\dots, N]$, $t\in[1,\dots,T]$.
% Formally, $x_{i,t}$ Granger causes $x_{j,t'}$ if there exists a point
% $x'_{i,t}\neq x_{i,t}$, such that $f_{j,t'}(\mathbf{X}')\neq f_{j,t'}(\mathbf{X})$, where $\mathbf{X}'$ is obtained by replacing $x_{i,t}$ in $\mathbf{X}$ with $x'_{i,t}$.
% \end{definition}
% \noindent 
% Essentially, if $x_{i,t}$ has impact on the prediction $\hat{y}_{j,t'}=f_{j,t'}(\mathbf{X})$, then we say the point $(i,t)$ Granger causes the point $(j,t')$.
% For imputation, we do not force $t<t'$ as in the forecasting setting, since to reconstruct the point $(i,t)$, the common practice is to consider data points both before and after $t'$ at the same time.

% \textbf{Frontdoor Adjustment.}
Given an incomplete spatiotemporal time series $\mathbf{G}=(\mathbf{X},\mathbf{A},\mathbf{M})$, 
% to impute a data point at $(i,t)$, where $i\in[1,\cdots, N]$ and $t\in[1,\cdots, T]$,
a standard deep imputation model $f=f^D\circ f^E$, where $f^E$, $f^D$ are the encoder and decoder, works as follows: 
(1) $f^E$ extracts embeddings $\mathcal{H}=\{\mathbf{h}_{i,t}\}_{i=1,t=1}^{N,T}$ from $\mathbf{G}$, 
(2) $f^D$ generates predictions $\{\hat{y}_{i,t}\}_{i=1,t=1}^{N,T}$ based on $\mathcal{H}$ to recover $\mathcal{Y}=\{{y}_{i,t}\}_{i=1,t=1}^{N,T}$.  
The model $f$ is trained by a reconstruction error e.g., MAE or RMSE.
Since minimizing MAE (or RMSE) is equivalent to maximizing the log-likelihood of Laplace (or Gaussian) distribution \cite{hodson2022root}, and thus we can view the objective of spatiotemporal imputation as maximizing $P(\mathcal{Y}|\mathcal{G})$.
Most existing studies focus on maximizing $P(\mathcal{Y}|\mathcal{G})$ yet few discuss the cause-and-effect relationship between $\mathcal{G}$ and $\mathcal{Y}$.
In this paper, we study their causality based on the Structure Causal Model (SCM)~\cite{pearl2018book}.

\begin{figure}
    \centering
    \includegraphics[width=0.28\textwidth]{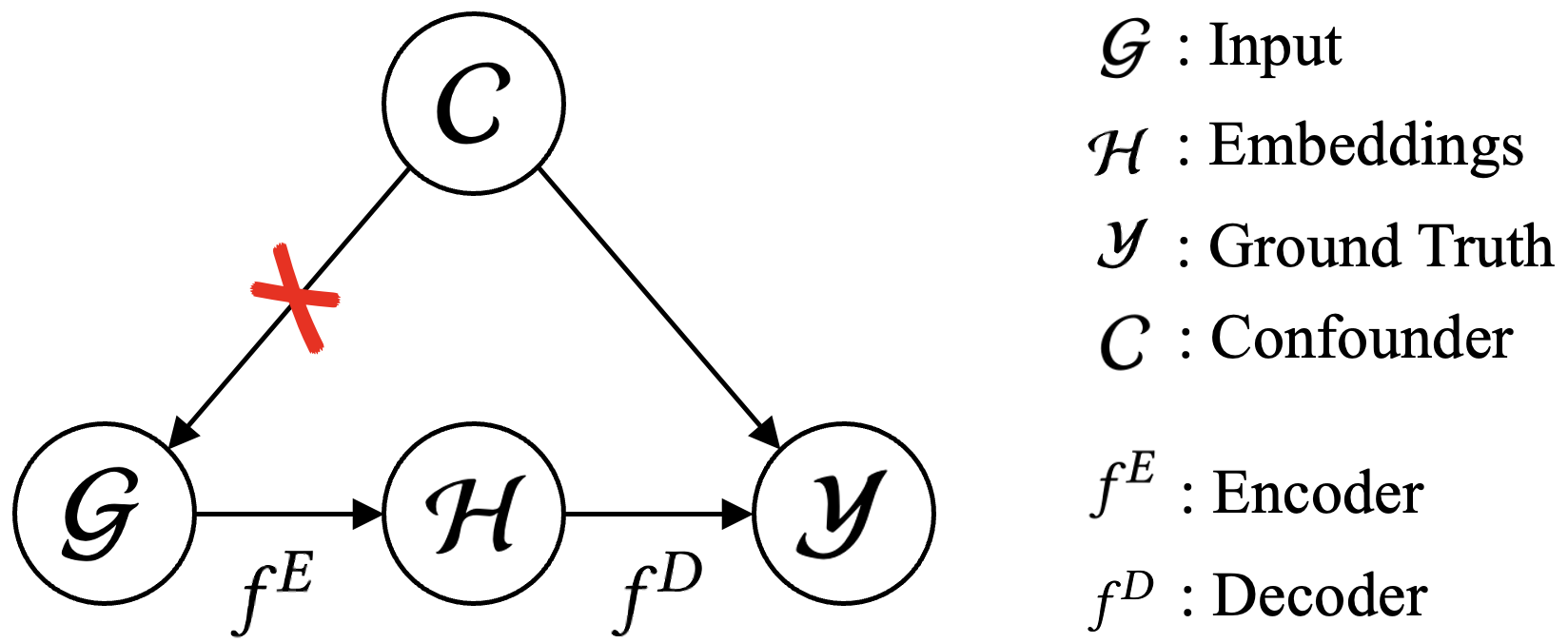}
    \caption{The Structure Causal Model (SCM) for spatiotemporal imputation. 
    The frontdoor adjustment removes the edge between confounder $\mathcal{C}$ and input $\mathcal{G}$.}
    \label{fig:scm}
\end{figure}

\paragraph{Structure Causal Model}
During data collection, it is inevitable that some unknown confounders $\mathcal{C}$ are included in datasets, which influence both $\mathcal{G}$ and $\mathcal{Y}$.
For example, sensors might record random background noise, and the constructed sensor network might contain shortcut edges.
The undesired information might bridge the input $\mathcal{G}$ and the output $\mathcal{Y}$ with spurious correlations,
which could lead to overfitting and make the model error-prone.
% could have negative impact on model performance [].
% bias \& error-prone
% some motivation here
% xxxxx
% maybe need experiments to corroborate 
% xxxxx

The causal relationship between $\mathcal{G}$, $\mathcal{H}$, $\mathcal{Y}$ and $\mathcal{C}$ can be modeled by the Structure Causal Model (SCM) \cite{pearl2018book} shown in Figure~\ref{fig:scm}. 
First, it is evident that $\mathcal{C}$ and $\mathcal{H}$ are not d-separable \cite{murphy2012machine}, since $\mathcal{C}$ can reach $\mathcal{H}$ via the path $\mathcal{C}\rightarrow\mathcal{G}\rightarrow\mathcal{H}$. 
This means that $\mathcal{H}$ and $\mathcal{C}$ are not independent and thus $\mathcal{H}$ contains information of $\mathcal{C}$.
Second, besides $\mathcal{G}\rightarrow\mathcal{H}\rightarrow\mathcal{Y}$, $\mathcal{C}$ introduces backdoor paths between $\mathcal{G}$ and $\mathcal{Y}$, as well as $\mathcal{H}$ and $\mathcal{Y}$: $\mathcal{G}\leftarrow\mathcal{C}\rightarrow\mathcal{Y}$, $\mathcal{H}\leftarrow\mathcal{G}\leftarrow\mathcal{C}\rightarrow\mathcal{Y}$.
% If the backdoor path is a lower-cost shortcut, then the model $f$ will not struggle to learn according to the desired path $\mathcal{G}\rightarrow\mathcal{H}\rightarrow\mathcal{Y}$, but will rather learn to capture the non-causal relationship via $\mathcal{C}$.
The model $f$ might take advantage of the backdoor paths to make decisions instead of struggling to discover the real cause-and-effect relationships  \cite{ sui2022causal,yang2021causal}.
Our goal is to eliminate the backdoor paths.
% introduced by $\mathcal{C}$.

% The corresponding graphical model of the above spatiotemporal imputation process is shown in the upper part of Figure \ref{fig:scm}.
% Since minimizing MAE (or RMSE) is equivalent to maximizing the log-likelihood of Laplace (or Gaussian) distribution~\cite{hodson2022root}, and thus we can re-write the reconstruction objective in the form of probabilities.
% Denote $\mathcal{G}_M$, as the variable for $\mathbf{G}_M$, based on the graphical model in the upper part in Figure \ref{fig:scm}, we have:
% \begin{equation}
%     p(\mathcal{Y}|\mathcal{G}_M) = p(\mathcal{Y}|\mathcal{H})p(\mathcal{H}|\mathcal{G}_M)
% \end{equation}

% \begin{equation}\label{eq:p_y_g}
%     P(\mathcal{Y}|\mathbf{G}_M) = \sum_{i,t} P(\mathbf{h}_{i,t}|\mathcal{G}_M)P(\mathcal{Y}|\mathbf{h}_{i,t})
% \end{equation}
% where $P(\mathbf{h}_{i,t}|\mathcal{G}_M)$ and $P(\mathcal{Y}|\mathbf{h}_{i,t})$ correspond to the encoder $f^E$ and decoder $f^D$ respectively.
% Most prior studies focus on maximizing the probability in Equation \eqref{eq:p_y_g}, yet few have discussed the cause-and-effect relationship between $\mathcal{G}_M$ and $\mathcal{Y}$.

\paragraph{Frontdoor Adjustment}
In statistics \cite{murphy2012machine}, a simple way to exclude the variable $\mathcal{C}$ in the SCM in Figure~\ref{fig:scm} is to marginalize it out.
However, marginalization requires $\mathcal{C}$ to be observable and measured by the marginal distribution $P(\mathcal{C})$, but in spatiotemporal imputation, $\mathcal{C}$ is usually unknown and difficult to measure.
Rather than directly marginalizing $\mathcal{C}$ out, we resort to the frontdoor adjustment \cite{pearl2018book}, which uses Pearl's do-calculus \cite{pearl2018book} to block the backdoor paths.
We follow the three steps of the frontdoor adjustment as follows.
\begin{enumerate}
    \item[(1)] \emph{Remove the backdoor path from $\mathcal{G}$ to $\mathcal{H}$.} 
    Given $\mathcal{G}$, there is no backdoor path from $\mathcal{G}$ to $\mathcal{H}$. Note that $\mathcal{G}$ cannot reach $\mathcal{H}$ via $\mathcal{G}\leftarrow\mathcal{C}\rightarrow\mathcal{Y}\leftarrow\mathcal{H}$ according to the d-separation theory \cite{murphy2012machine}.
    Therefore, we have:
    \begin{equation}
        P(\mathcal{H}|do(\mathcal{G}))=P(\mathcal{H}|\mathcal{G}).
    \end{equation}
    \item[(2)] \emph{Remove the backdoor path from $\mathcal{H}$ to $\mathcal{Y}$.}
    There is a backdoor path between $\mathcal{H}$ and $\mathcal{Y}$: $\mathcal{H}\leftarrow\mathcal{G}\leftarrow\mathcal{C}\rightarrow\mathcal{Y}$. This backdoor path can be blocked by marginalizing out $\mathcal{G}$:
    \begin{equation}
        P(\mathcal{Y}|do(\mathcal{H})) = \sum_{\mathbf{G}}P(\mathcal{Y}|\mathcal{H},\mathbf{G}_M)P(\mathbf{G}).
    \end{equation}
    % \item[3)] \emph{Identify the causal effect of $\mathcal{G}$ on $\mathcal{Y}$ by combining the results of the above two steps:}
    \item[(3)] \emph{Combine the results of the above two steps:}
    \begin{equation}\label{eq:front_door}
    \begin{split}
        P(\mathcal{Y}|do(\mathcal{G})) &= \sum_{i,t}P(\mathbf{h}_{i,j}|do(\mathcal{G}))P(\mathcal{Y}|do(\mathbf{h}_{i,t}))\\
        &=\sum_{i,t}P(\mathbf{h}_{i,t}|\mathcal{G})\sum_{\mathbf{G}}P(\mathcal{Y}|\mathbf{h}_{i,t},\mathbf{G})P(\mathbf{G}). 
    \end{split}
    \end{equation}
\end{enumerate}
In general, $P(\mathbf{h}_{i,t}|\mathcal{G})$ can be viewed as the encoder $f^E$, and the rest part, including $\sum_{\mathbf{G}}P(\mathcal{Y}|\mathbf{h}_{i,t}, \mathcal{G})P(\mathbf{G})$ and the sum over all data points $\sum_{i,t}$, can be viewed as the decoder $f^D$.
% Besides, $P(\mathbf{G})$ is the data distribution.
% In the next subsection, we introduce a novel Causality-Aware Spatiotemporal Graph Neural Network (\casper) based on the frontdoor adjusted causal probability in Equation \eqref{eq:front_door}.
In the next subsection, we show how to implement Equation \eqref{eq:front_door}.

% \subsection{Causality-Aware Spatiotemporal Graph Neural Network}
\subsection{Architecture of \casper}
% In this subsection, we introduce a novel \casper\ based on Equation~\ref{eq:front_door}.
% We first present an overview of \casper\ and then elaborate the two most important components spatiotemporal causal attention and the prompt based decoder.
In this subsection, based on Equation \eqref{eq:front_door}, we propose a novel \underline{C}usality-\underline{A}ware \underline{Sp}atiot\underline{e}mpo\underline{r}al Graph Neural Network (\casper).
We first present an overview of \casper, consisting of a Prompt Based Decoder (PBD) and an encoder with Spatiotemporal Causal Attention (SCA).
Next, we elaborate PBD and SCA in detail.

\begin{figure}
    \centering
    \includegraphics[width=.35\textwidth]{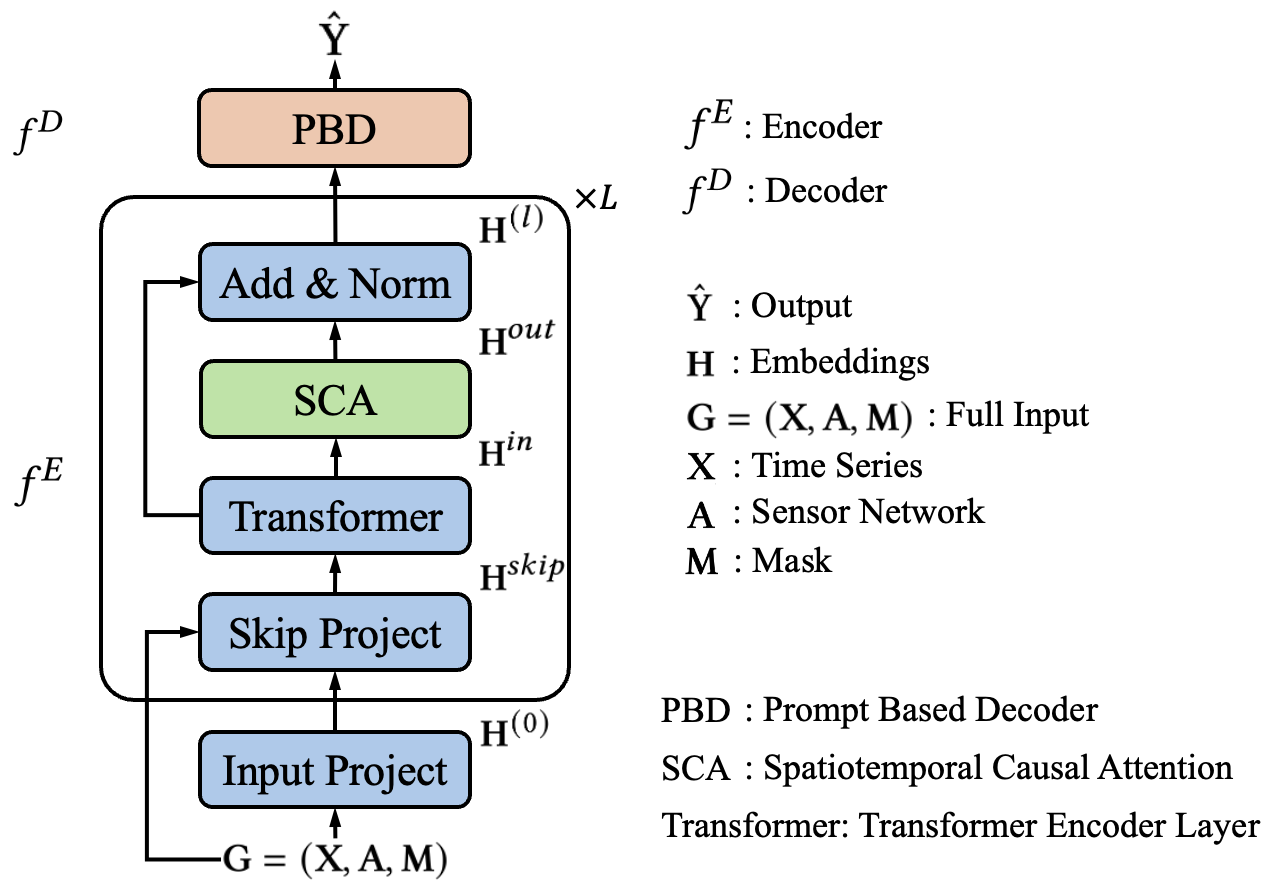}
    \caption{Overview of \casper.}
    \label{fig:overview}
\end{figure}

\paragraph{Overview}
Figure \ref{fig:overview} shows an overview of \casper.
The encoder $f^E$ is comprised of an \emph{input project} and $L$ layers of the combination of \emph{skip project}, \emph{transformer}, \emph{SCA} and \emph{add \& norm}.
Let $\mathbf{m}\in\mathbb{R}^{d}$ be the embedding for the missing points.
The \emph{input project} module encodes the raw input $\mathbf{G}=(\mathbf{X}, \mathbf{A}, \mathbf{M})$ into $\mathbf{H}^{(0)}$ via:
\begin{equation}
    \mathbf{H}^{(0)}=\text{MLP}(\mathbf{X})\odot\mathbf{M} + \mathbf{m}\odot(1-\mathbf{M}),
\end{equation}
where MLP stands for Multi-Layer Perceptron and $\odot$ denotes the Hadamard product.
The \emph{skip project} module prevents gradient vanishing and improves the performance by injecting the $\mathbf{G}$ into the embeddings from the previous layer $\mathbf{H}^{(l-1)}$:
\begin{equation}
    \mathbf{H}^{skip}=\mathbf{H}^{(l-1)} + \text{MLP}(\mathbf{X})\odot\mathbf{M} + \mathbf{m}\odot(1-\mathbf{M}).
\end{equation}
The \emph{transformer} encoder layer \cite{vaswani2017attention} learns temporal information for each time series within $\mathbf{H}^{skip}$: 
\begin{equation}
    \mathbf{H}^{in}=\text{Transformer}(\mathbf{H}^{skip}).
\end{equation}
\emph{SCA} discovers spatiotemporal causal relationships among embeddings based on $\mathbf{A}$, and encodes causal information into embeddings:
\begin{equation}
    \mathbf{H}^{out} = \text{SCA}(\mathbf{H}^{in}, \mathbf{A}).
\end{equation}
The final embeddings of the $l$-th layer are given by:
\begin{equation}
    \mathbf{H}^{(l)} = \text{LayerNorm}(\mathbf{H}^{in} + \mathbf{H}^{out}).
\end{equation}
Given the embeddings $\mathbf{H}=\mathbf{H}^{(L)}$ obtained by the encoder $f^E$, the \emph{PBD} module in $f^D$ generates the predictions $\hat{\mathbf{Y}}=\{\hat{y}_{i,t}\}_{i=1,t=1}^{N,T}$:
\begin{equation}
    \hat{\mathbf{Y}} = \text{PBD}(\mathbf{H}).
\end{equation}
% Denote $\mathcal{Q}=\mathbf{h}_{i,t}^{in}$ as the query, its context keys as $\mathcal{K}=\{\mathbf{h}_{i',t'}\}$ and the message from neighbors to the point $(i,t)$ as $\mathcal{V}=\{\mathbf{v}_{i',t'}\}$, where $i'\in\mathcal{N}(i)$, $t'\in[1,\cdots,T]$ and $\mathbf{v}_{i',t'}=\text{MLP}([\mathbf{h}_{i,t}^{in};\mathbf{h}_{i',t'}^{in}])$.
% Next, we first introduce PBD and then introduce SCA.

\begin{figure}
    \centering
    \includegraphics[width=.25\textwidth]{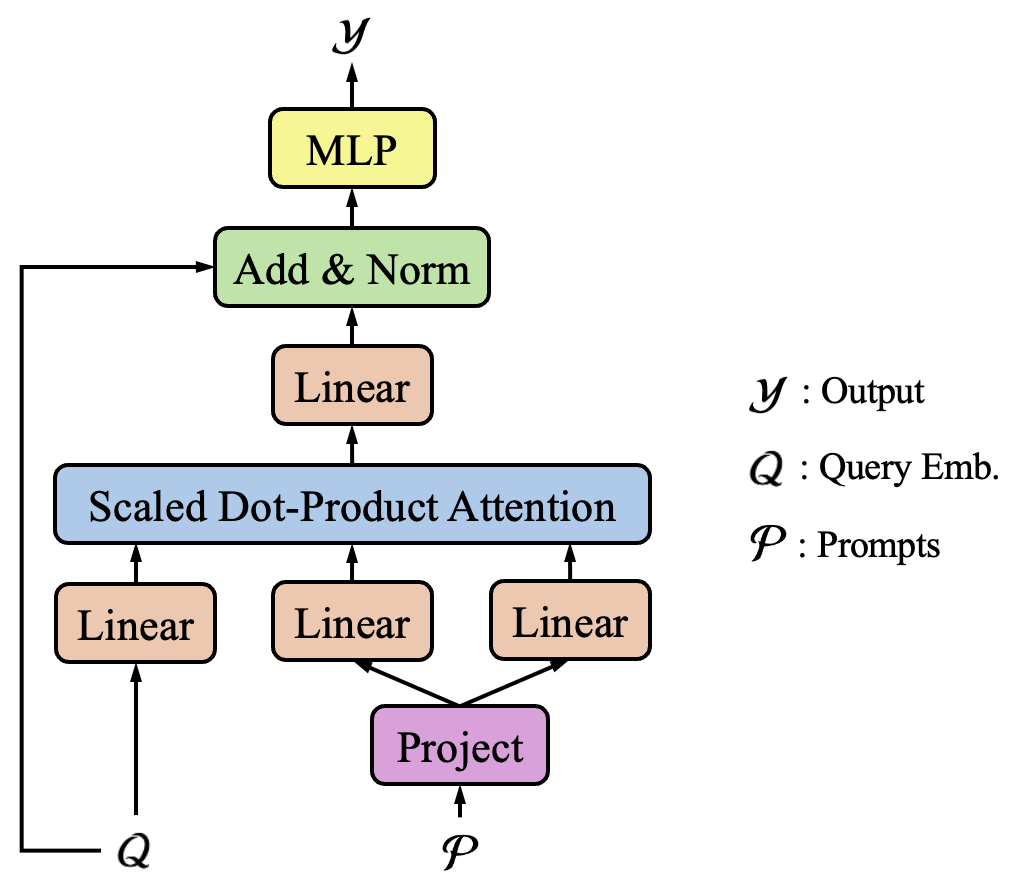}
    \caption{Prompt Based Decoder (PBD).}
    \label{fig:pbd}
\end{figure}

\paragraph{Prompt Based Decoder}
Suppose we are given the input $\mathbf{G}$ and the target is to recover $\mathcal{Y}=y_{i,t}$. 
In Equation \eqref{eq:front_door}, the decoder $f^D$ is comprised of (1) a sum over all possible $\mathbf{G}'$ in the dataset $\sum_{\mathbf{G}'}P(y_{i,t}|\mathbf{h}_{i',t'}, \mathbf{G}')P(\mathbf{G}')$ and (2) a sum over all data points $\sum_{i',t'}$ in $\mathbf{G}$.
For (2), since the encoders nowadays, e.g., Transformer \cite{vaswani2017attention}, are so powerful that could encode sufficient context information $\mathbf{G}$ in $\mathbf{h}_{i,t}$, and thus the decoder $f^D$ could only take $\mathbf{h}_{i,t}$ as input~\cite{marisca2022learning}, instead of all possible $\mathbf{h}_{i',t'}$.
Therefore, we could drop $\sum_{i',t'}$ and only implement $\sum_{\mathbf{G}'}P(y_{i,t}|\mathbf{h}_{i,t}, \mathbf{G}')P(\mathbf{G}')$.
% Now, $f^D$ corresponds to
% \begin{equation}\label{eq:fd}
%     \sum_{\mathbf{G}'}P(y_{i,t}|\mathbf{h}_{i,t}, \mathbf{G}')P(\mathbf{G}')
% \end{equation}
% The key challenge for implementing Equation \eqref{eq:fd} is how to implement the sum over $\mathbf{G}'$.

Now, the challenge is how to implement the sum over $\mathbf{G}'$.
Simple solutions include randomly sampling a set from the training data, or clustering $\mathbf{G}'$ into $K$ clusters and using the cluster centers as an approximation.
However, random sampling could be unstable in practice, and clustering requires extra pre-trained models to extract embeddings of $\mathbf{G}'$ in advance.
Inspired by \cite{jia2022visual}, which uses prompts to capture the context information of the downstream task, we introduce a Prompt Based Decoder (PBD) to automatically capture the global context information of the dataset during model training.

An illustration of the proposed PBD is shown in Figure \ref{fig:pbd}.
$\mathcal{Q}=\mathbf{h}_{i,t}$ is the query, $\mathcal{P}=\{\mathbf{p}_{n}\}_{n=1}^{N_P}$ is the set of learnable prompts, which are randomly initialized embedding vectors.
In Figure \ref{fig:pbd}, the \emph{Project} is a linear function followed with a LayerNorm \cite{ba2016layer}.
Details of the scaled dot-product attention can be found in \cite{vaswani2017attention}, and it can be easily extended into the multi-head version as in \cite{vaswani2017attention}.

\paragraph{Spatiotemporal Causal Attention}
% According to Equation \eqref{eq:front_door}
% PBD has implemented the most critical part of the frontdoor adjustment shown in Equation~\eqref{eq:front_door}.
% As shown in Figure \ref{fig:scm} and Equation \eqref{eq:front_door}, the frontdoor adjustment \cite{pearl2018book} could eliminate the impact of the unknown confounders $\mathcal{C}$ and ensure the causality between the input $\mathcal{G}$ and the output $\mathcal{Y}$ by summing over $\mathbf{G}$, $i$ and $t$.
% by forcing the model to focus on the causal relationships \cite{yang2021causal}.
% Pearl's do-calculus , however, it does not explicitly model the causal relationship among embeddings.
% As shown in 
% Although it could discover the causal relationships among embeddings to a certain degree, 
% A natural question arises that how well can the model discover the causal relationships.
% Our experiments show that 
Attention functions (Definition \ref{def:attention}) have become indispensable in deep learning models \cite{vaswani2017attention,jing2022x}, which could effectively capture the context information for the target embedding.
Although attention scores could show the correlation between embeddings, correlation does not necessarily imply causality and thus sometimes could induce undesired non-causal information into embeddings \cite{sui2022causal}.
Based on the frontdoor adjustment, PBD could eliminate the impact of unknown confounders $\mathcal{C}$ and ensure the causality between the input $\mathcal{G}$ and the output $\mathcal{Y}$ by summing over $\mathbf{G}$, $i$ and $t$.
However, it guides the attention functions to discover the causal relationships at a high level, and thus the learned causal relationships, i.e., attention scores, might still be dense and a little bit difficult to interpret (see Figuers \ref{fig:softmax_4_1} and \ref{fig:softmax_4_2}).
% It guides the causal discovery among the input and out at a high level, which treats the encoder $f^E$ or $P(\mathbf{h}_{i,t}|\mathbf{G})$ as a black box.
% Without direct guidance for the encoder, its learned causal might still be dense and a little bit difficult to interpret.
% To further guide the model to discover the sparse causal relationships, we  Granger causality \cite{granger1969investigating} (Definition \ref{def:granger_causality}), which provides 
To directly guide the model to discover the sparse causal relationships, we first define the causality for embeddings (Definition \ref{def:granger_causality_emb}) based on the Granger causality \cite{granger1969investigating} (Definition \ref{def:granger_causality}), and then introduce a novel Spatiotemporal Causal Attention (SCA) module to discover the sparse causality between embeddings.

\begin{definition}[Unconstrained Granger Causality for Embeddings]\label{def:granger_causality_emb}
Denote the target embedding as $\mathcal{Q}=\mathbf{h}_{i,t}^{in}$ and the set of context embeddings as $\mathcal{K}=\{\mathbf{h}_{i',t'}^{in}\}_{i'=1,t'=1}^{N,T}$.
Let $\mathbf{h}_{i,t}^{out}=f_{i,t}(\mathcal{Q}; \mathcal{K})$ be an embedding updating function, e.g., attention function.
If there is a $\mathbf{h}_{i',t'}^{in}\in\mathcal{K}$ s.t. changing the value of $\mathbf{h}_{i',t'}^{in}$ will change the value of $\mathbf{h}_{i,t}^{out}$, then $\mathbf{h}_{i',t'}^{in}$ Granger causes $\mathbf{h}_{i,t}^{out}$.
\end{definition}
We do not strictly enforce the time $t'\leq t$ for the imputation task as for the forecasting task (Definition \ref{def:granger_causality}), since 
(1) a missing value could appear at the beginning of the input time series segment, and there are no prior points available; 
(2) most imputation methods in the literature consider both past $t'\leq t$ and future $t' > t$ reference points $\mathbf{h}_{i',t'}$ for the data point to be imputed $\mathbf{h}_{i,t}$; 
(3) given the learned weight $w_{i',t'}$ between $\mathbf{h}_{i',t'}$ and $\mathbf{h}_{i,t}$, it is {easy} to distinguish whether it is from the past or future by comparing $t'$ and $t$, and thus we can easily obtain the time-constrained causal graph if necessary.

Let $f_{i,t}$ be an attention function as shown in Definition \ref{def:attention}:
\begin{equation}
    \mathbf{h}_{i,t}^{out} = \sum_{i'=1,t'=1}^{N,T}\alpha_{i',t'}\mathbf{v}_{i',t'},\quad \mathbf{v}_{i',t'}=f_v(\mathbf{h}_{i,t}^{in},\mathbf{h}_{i't'}^{in}),
\end{equation}
where $\alpha_{i',t'}$ is the attention weight, and $f_v$ is the message function.
According to Definition \ref{def:granger_causality_emb}, 
if $\alpha_{i',t'}\neq0$, then $\mathbf{h}_{i',t'}^{in}$ Granger causes $\mathbf{h}_{i,t}^{out}$;
otherwise $\mathbf{h}_{i',t'}^{in}$ does not Granger cause $\mathbf{h}_{i,t}^{out}$.
In practice, without directly manipulating, $\alpha_{i',t'}>0$ holds for many noisy messages $\mathbf{v}_{i',t'}$, as shown in Figure \ref{fig:softmax_4_1}-\ref{fig:softmax_2_1} in our experiments.
\begin{figure}[t]
    \centering
    \includegraphics[width=0.35\textwidth]{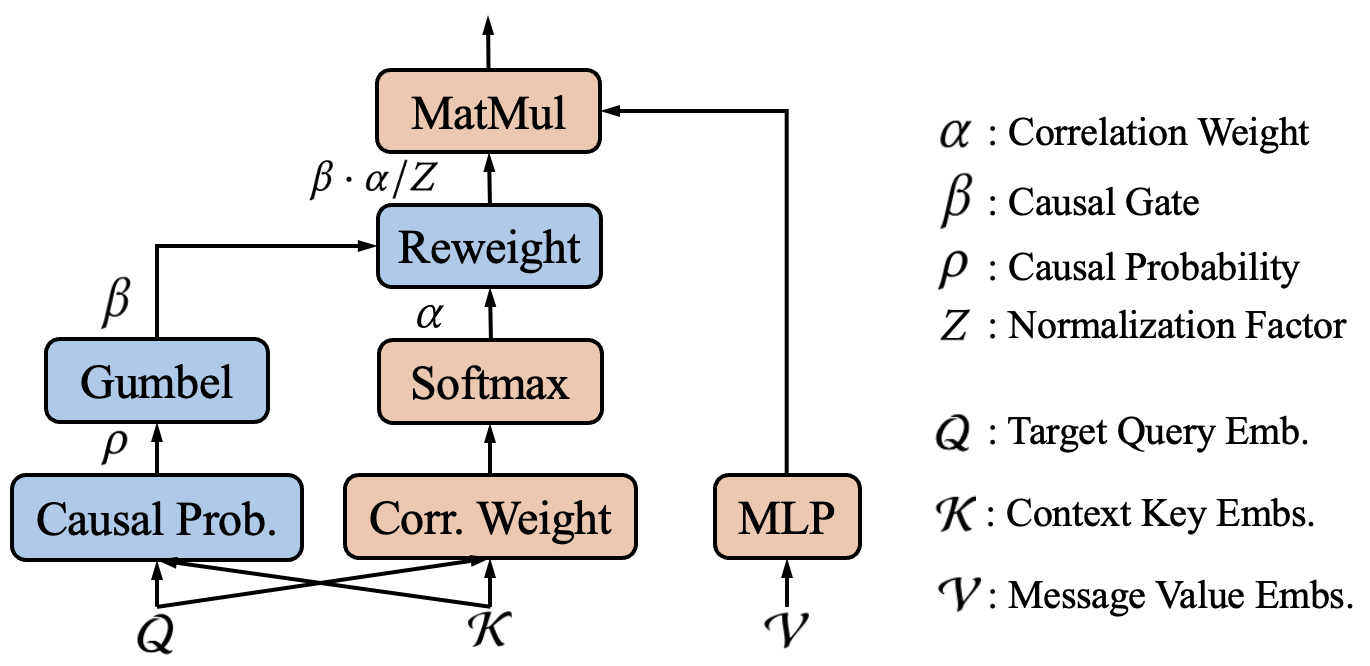}
    \caption{Spatiotemporal Causal Attention (SCA). Corr. Weight and Causal Prob. correspond to Equation \eqref{eq:corrw}\eqref{eq:rho1}. }
    % $\mathcal{N}(i)$ denotes the set of neighbors for the $i$-th sensor.}
    \label{fig:sca}
\end{figure}
% Attention functions (Definition \ref{def:attention}) have become indispensable in deep learning models \cite{vaswani2017attention}, which could effectively capture the context information for the target embedding.
% Although attention scores could show the correlation between embeddings, yet correlation does not necessarily imply causality and thus sometimes could induce undesired non-causal information into embeddings \cite{sui2022causal}.
% In this paper, we propose a novel Spatiotemporal Causal Attention (SCA) to capture causal relationships between embeddings.
% As shown in Figure \ref{fig:sca}, SCA is comprised of two components: (1) a spatiotemporal graph attention function (orange), which learns the correlation between embeddings; (2) a causal gate (blue), which discovers the causal and non-causal relationships.
To further enforce the weights of the noisy points to be zero and discover Granger causality, we propose a novel SCA.
As shown in Figure \ref{fig:sca}, SCA is comprised of two components: (1) a spatiotemporal graph attention function (orange), which learns the correlation between embeddings; (2) a causal gate (blue), which discovers the causal and non-causal relationships.

Let $\mathcal{Q}=\mathbf{h}_{i,t}^{in}$ be the target query embedding, and $\mathcal{N}(i)$ be the neighbors of the $i$-th sensor, i.e., $\mathbf{A}[i,i']\neq0$, $\forall i'\in\mathcal{N}(i)$.
Denote the context keys as $\mathcal{K}=\{\mathbf{h}_{i',t'}^{in}\}_{i',t'}$ and the message values from $\mathcal{N}(i)$ to the point $(i,t)$ as $\mathcal{V}=\{\mathbf{v}_{i',t'}\}_{i',t'}$, where $i'\in\mathcal{N}(i)$, $t'\in\{1,\cdots,T\}$ and $\mathbf{v}_{i',t'}=\text{MLP}([\mathbf{h}_{i,t}^{in};\mathbf{h}_{i',t'}^{in}])$.
We define SCA as:
\begin{equation}\label{eq:sca}
    \mathbf{h}_{i,t}^{out} = \frac{1}{Z}\sum_{i'\in\mathcal{N}(i)}\sum_{t'=1}^T\beta_{i',t'}\alpha_{i',t'}\mathbf{v}_{i',t'},
\end{equation}
where $Z\in\mathbb{R}$ is a normalization factor, $\alpha_{i',t'}$ is the correlation weight, $\beta_{i',t'}\sim \text{Bernoulli}(\rho_{i',t'})$ is the causal gate, and $\rho_{i',t'}$ is the probability that $\mathbf{h}_{i',t'}^{in}$ Granger causes $\mathbf{h}_{i,t}^{out}$.
According to Definition~\ref{def:granger_causality_emb}, if $\beta_{i',t'}\cdot\alpha_{i',t'}>0$, then $\mathbf{h}_{i',t'}^{in}$ Granger causes $\mathbf{h}_{i,t}^{out}$; 
% if $\beta_{i',t'}\cdot\alpha_{i',t'}=0$, then there is no Granger causality between $\mathbf{h}_{i',t'}^{in}$ and $\mathbf{h}_{i,t}^{out}$.
otherwise, there is no Granger causality between $\mathbf{h}_{i',t'}^{in}$ and $\mathbf{h}_{i,t}^{out}$.
The function of calculating correlation weight $\alpha_{i',t'}$ is given by:
\begin{align}
    \alpha_{i't'}&=\frac{\exp(s(\mathbf{h}_{i,t};\mathbf{h}_{i',t'}))}{\sum_{j\in\mathcal{N}(i)}\sum_{r=1}^T\exp(s(\mathbf{h}_{i,t};\mathbf{h}_{j,r}))}\label{eq:corrw},\\
    s(\mathbf{h}_{i,t}^{in};\mathbf{h}_{i',t'}^{in}) &= {(\mathbf{W}_Q\mathbf{h}_{i,t}^{in})^T(\mathbf{W}_K\mathbf{h}_{i',t'}^{in}})/{\sqrt{d}},
\end{align}
where $\mathbf{W}_Q,\mathbf{W}_K\in\mathbb{R}^{d\times d}$ are learnable weights, and $d$ is the size of the hidden dimension.
We build a neural network to learn the probability $\rho_{i',t'}$ of the causal gate $\beta_{i',t'}$:
\begin{align}
    % \mathbf{c}_{i',t'} &= \text{MLP}_1(\mathbf{h}_{i,t}^{in}) + \text{MLP}_2(\mathbf{h}_{i',t'}^{in}) \\
    % \rho_{i',t'} &= \frac{\exp(\mathbf{c}_{i',t'})}{1+\exp(\mathbf{c}_{i',t'})}
    % {c}_{i',t'} &= \mathbf{W}_c[\mathbf{W}_{Q_c}\mathbf{h}_{i,t}^{in};\mathbf{W}_{K_c}\mathbf{h}_{i',t'}^{in}]\label{eq:rho1}\\
    % \rho_{i',t'} = \sigma(\text{MLP}_1(\mathbf{h}_{i,t}^{in}) + \text{MLP}_2(\mathbf{h}_{i',t'}^{in}))
    % \rho_{i',t'} &= \sigma({c}_{i',t'})\label{eq:rho2}
    \rho_{i',t'} = \sigma(\mathbf{W}_c[\mathbf{W}_{Q_c}\mathbf{h}_{i,t}^{in};\mathbf{W}_{K_c}\mathbf{h}_{i',t'}^{in}]),\label{eq:rho1}
\end{align}
where $\mathbf{W}_c\in\mathbb{R}^{1\times 2d},\mathbf{W}_{Q_c},\mathbf{W}_{K_c}\in\mathbb{R}^{d\times d}$ are learnable weights, and $\sigma$ is the Sigmoid activation function.

There are two practical issues of directly using $\rho$ in the above equation. 
First, the sampling operation $\beta_{i',t'}\sim \text{Bernoulli}(\rho_{i',t'})$ is in-differentiable.
To address this issue,
we use the differentiable re-parameterization technique Gumbel-Softmax \cite{jang2016categorical} to obtain $\beta_{i',t'}$:
\begin{equation}\label{eq:gumbel}
    % \beta_{i',t'} = \text{Gumbel}(\text{MLP}_1(\mathbf{h}_{i,t}^{in}) + \text{MLP}_2(\mathbf{h}_{i',t'}^{in}))
    % \beta_{i',t'} = \text{Gumbel}(\mathbf{c}_{i',t'})
    \beta_{i',t'} = \frac{\exp((\log\rho_{i',t'}+g)/\tau)}{\exp((\log\rho_{i',t'}+g)/\tau) + \exp((\log(1-\rho_{i',t'})+g)/\tau)},
\end{equation}
where $g=-\log(-\log(u))$, $u\sim\text{Uniform}(0,1)$, and $\tau$ is the temperature parameter.

Second, if $\rho_{i',t'}$ is not close to 0 or 1, the model's decision could be ambiguous during inference.
% $\beta_{i',t'}$ can be either 0 or 1 for the same input data. During inference, 
For example, if $\rho_{i',t'}=0.2$, then for the same input data, for 20\% time, the model shows $\mathbf{h}_{i',t'}^{in}$ Granger causes $\mathbf{h}_{i,t}^{out}$, and for the other 80\% time, the model shows $\mathbf{h}_{i',t'}^{in}$ does not Granger cause $\mathbf{h}_{i,t}^{out}$. 
To avoid such an ambiguous situation, we enforce $\rho_{i',t'}\rightarrow0/1$ by placing the $l_1$ regularization over $\rho_{i',t'}$.

It can be theoretically proven that $\rho_{i',t'}$ will converge to 0 or 1 (see Section \ref{sec:theorem}).
Additionally, in practice, the correlation weight $\alpha$ can be easily extended to the multi-head version as in \cite{vaswani2017attention}.

\paragraph{Loss Function}\label{sec:loss}
\casper\ is trained by the masked MAE.
For a given spatiotemporal time series segment with $N$ nodes and $T$ length, the loss is defined as:
\begin{equation}\label{eq:loss}
    % \mathcal{L} = \frac{(1-\mathbf{M})\odot}{||1-\mathbf{M}||_{nnz}}||(\hat{\mathbf{Y}}-\mathbf{Y})||_1 + \lambda||\rho||_1
    \mathcal{L} = \sum_{i=1,t=1}^{N,T}m_{i,t}\cdot|{y}_{i,t}-\hat{y}_{i,t}| + \lambda||\Phi||_1,
\end{equation}
% where $\mathbf{Y}$ and $\hat{\mathbf{Y}}$ are ground-truth and predicted values, $\mathbf{M}$ is the mask, $\odot$ is the Hadamard product, $||\cdot||_1$ and $||\cdot||_{nnz}$ denote the $L1$ norm and the number of non-zeros elements.
where $m_{i,t}$ is the mask, $y_{i,t}$ is the ground-truth value, $\hat{y}_{i,t}$ is the predicted value, 
$\Phi$ is the set of all $\rho$ in SCA, $\lambda$ is a tunable coefficient, and $||\cdot||_1$ denotes the $l_1$ norm.

\subsection{Framework Analysis}
In this subsection, we provide further analysis of the proposed \casper, including theoretical analysis and complexity analysis.

\paragraph{Theoretical Analysis}\label{sec:theorem}
We theoretically prove that $\rho_{i',t'}$ in Equation \eqref{eq:rho1} will converge to 0 or 1 in Theorem \ref{thm:convergence}, and thus $\rho_{i',t'}$ indicates the Granger causality.
If $\rho_{i',t'}=0$, then $\beta_{i',t'}\cdot\alpha_{i',t'}=0$, showing that $\mathbf{h}^{in}_{i',t'}$ does not Granger cause $\mathbf{h}^{out}_{i,t}$.
Moreover, from the proof of Theorem \ref{thm:convergence}, it can be observed that $\rho_{i',t'}$ is actually a gradient-based explanation (see Remark \ref{remark:rho}), which determines causal and non-causal relationships based on the gradients.
% Denote $g$ as the gradient at $\rho_{i',t'}$ (Equation \eqref{eq:g}).
% If $g>\lambda$, then $\rho_{i',t'}\rightarrow1$; if $g<\lambda$, then $\rho_{i',t'}\rightarrow0$, where $\lambda$ is the coefficient in Equation \eqref{eq:loss}.
Compared with the classic gradient-based explanation methods \cite{selvaraju2017grad}, 
which needs extra steps to calculate gradients after the model is trained,
the proposed $\rho$ has two advantages:
% There are two advantages of $\rho$ over classic gradient based explanation methods \cite{selvaraju2017grad}: 
(1) $\rho$ does not require extra steps for calculating derivatives, and the value of $\rho$ could directly provide the explanation;
(2) the parameters associated with $\rho$ are jointly trained with the model, and thus it can guide the model to focus on the most important relationships during training.

\begin{thm}[Convergence of $\rho$]\label{thm:convergence}
$\rho$ could converge to 0 or 1 by updating its parameters based on $\mathcal{L}$ in Equation \eqref{eq:loss}.
% When trained with $\mathcal{L}$ in Equation \eqref{eq:loss}, $\rho$ (Equation \eqref{eq:rho1}) converges to 0 or 1, with other parameters fixed, 
% Suppose we are going to update the parameters of $\rho$ (Equation \eqref{eq:rho1}) with all other parameters fixed, then $\rho$ could converge to 0 or 1.
\end{thm}
\begin{proof}[Proof Sketch]
% Due to space limitations, we only provide a proof sketch here.
% for more details, please refer to Appendix.
For simplicity, let's only consider the loss for a single point $(i,t)$, where $m_{i,t}=1$ and ${y}_{i,t} - \hat{y}_{i,t}>0$:
\begin{equation}\label{eq:loss_single}
    \mathcal{L}_{i,t} = {y}_{i,t} - \hat{y}_{i,t} + \lambda||\Phi||_1
\end{equation}
% We can calculate the gradient for any single element $w$ in Equation~\eqref{eq:rho1}.
Since Equation \eqref{eq:rho1} is essentially a linear function with a Sigmoid activation, we can rewrite it as:
\begin{equation}
    \rho_{i',t'}=\sigma(\mathbf{w}^T\mathbf{h}^{in}),\quad \mathbf{h}^{in}=[\mathbf{h}_{i,t}^{in};\mathbf{h}_{i',t'}^{in}], \quad\mathbf{w}\in\mathbb{R}^{2d}
\end{equation}
% where $\mathbf{h}^{in}=[\mathbf{h}_{i,t}^{in};\mathbf{h}_{i',t'}^{in}]$, $\mathbf{w}\in\mathbb{R}^{2d}$.
Let $w_j=\mathbf{w}[j]$ and $h_j^{in}=\mathbf{h}^{in}[j]$,
then the gradient at $w_j$ is:
\begin{align}
\begin{split}
    \frac{\partial \mathcal{L}_{i,t}}{\partial w_j} &= \frac{\partial{y}_{i,t} - \hat{y}_{i,t}}{\partial \mathbf{h}_{i,t}^{out}}\cdot\frac{\partial \mathbf{h}_{i,t}^{out}}{\partial \beta_{i',t'}}\cdot\frac{\partial \beta_{i',t'}}{\partial \rho_{i',t'}}\cdot\frac{\partial \rho_{i',t'}}{\partial w_j} + \lambda\frac{\partial \rho_{i',t'}}{w_j}\\
    % &= \mathbf{g}^T{\alpha_{i',t'}}\mathbf{v}_{i',t'}\cdot(-1)\cdot\rho_{i',t'}(1-\rho_{i',t'}) + \lambda(1-\rho_{i',t'})\\
    % &= (\lambda-{\alpha_{i',t'}}\mathbf{g}^T\mathbf{v}_{i',t'}/Z)\rho_{i',t'}(1-\rho_{i',t'})h^{in}_j
    &= (\lambda-g)\rho_{i',t'}(1-\rho_{i',t'})h^{in}_j
\end{split}\\
\begin{split}\label{eq:g}
    g&=\frac{\partial\hat{y}_{i,t}}{\partial \mathbf{h}_{i,t}^{out}}\cdot\frac{\partial \mathbf{h}_{i,t}^{out}}{\partial \beta_{i',t'}}\cdot\frac{\partial \beta_{i',t'}}{\partial \rho_{i',t'}}=\frac{\partial\hat{y}_{i,t}}{\partial \mathbf{h}_{i,t}^{out}}\cdot\frac{\alpha_{i',t'}}{Z}\mathbf{v}_{i',t'}\cdot1
\end{split}
\end{align}
where $\mathbf{h}_{i,t}^{out},\beta_{i',t'}$ are from Equation \eqref{eq:gumbel}.
% where $\mathbf{g}=-\frac{\partial\mathcal{L}}{\partial \mathbf{h}_{i,j}^{out}}\mathbf{v}_{i',t'}$, 
In gradient descent, the updating function of $w_j$ is:
\begin{equation}\label{eq:sgd}
% \begin{split}
%     % \mathbf{w}[j]^{(k+1)} &= \mathbf{w}[j]^{(k)} \\
%     % &- \eta(\lambda-{\alpha_{i',t'}}\mathbf{g}^T\mathbf{v}_{i',t'}/Z)\rho_{i',t'}(1-\rho_{i',t'})\mathbf{h}^{in}[j]
%     w_j^{(k+1)} = w_j^{(k)} - \eta(\lambda-{\alpha_{i',t'}}\mathbf{g}^T\mathbf{v}_{i',t'}/Z)\rho_{i',t'}^{(k)}(1-\rho_{i',t'}^{(k)})h^{in}_j
% \end{split}
w_j^{(k+1)} = w_j^{(k)} - \eta(\lambda-g)\rho_{i',t'}^{(k)}(1-\rho_{i',t'}^{(k)})h^{in}_j
\end{equation}
where $k$ is the iteration index and $\eta>0$ is the learning rate. 
As we only consider the parameters $\mathbf{w}$ of $\rho$, for simplicity, let's fix all other parameters in the model.
% Since all parameters are fixed except for $\mathbf{w}$, thus ${\partial\hat{y}_{i,t}}/{\partial \mathbf{h}_{i,t}^{out}},\mathbf{v}_{i',t'},\alpha_{i',t'}$ are also fixed.
Therefore, ${\partial\hat{y}_{i,t}}/{\partial \mathbf{h}_{i,t}^{out}},\mathbf{v}_{i',t'},\alpha_{i',t'}$ are fixed.
The normalization factor $Z\leq1$, and $Z=1$ if and only if $\beta_{i',t'}=1$ for $\forall i',t'$.
$\rho_{i',t'}(1-\rho_{i',t'})>0$ since $\rho_{i',t'}\in(0,1)$. 
% Now, suppose $\lambda>{\alpha_{i',t'}}\mathbf{c}^T\mathbf{v}_{i',t'}/{Z}$, if $h_j^{in}>0$, then ${w}^{(k+1)}_j<{w}^{(k)}_j$; if $h_j^{in}<0$, then ${w}^{(k+1)}_j>{w}^{(k)}_j$.
Now, suppose $g>\lambda$, if $h_j^{in}>0$, then ${w}^{(k+1)}_j>{w}^{(k)}_j$. 
When $k\rightarrow\infty$, $w_j^{(k)}\rightarrow+\infty$.
Otherwise, if $h_j^{in}<0$, then ${w}^{(k+1)}_j<{w}^{(k)}_j$.
When $k\rightarrow\infty$, $w_j^{(k)}\rightarrow-\infty$.
Therefore, when $k\rightarrow\infty$, $\mathbf{w}^T\mathbf{h}\rightarrow+\infty$.
% \begin{equation}\label{eq:rhorho}
% \rho_{i',t'}^{(k+1)}=\sigma({\mathbf{w}^{(k+1)}}^T\mathbf{h}^{in}) > \sigma({\mathbf{w}^{(k)}}^T\mathbf{h}^{in})=\rho_{i',t'}^{(k)}
% \end{equation}
% where the equality holds if and only if $\rho_{i',t'}^{(k)}=0$.
As a result, $\rho_{i',t'}^{(k)}=\sigma({\mathbf{w}^{(k)}}^T\mathbf{h}^{in})\rightarrow1$, as $k\rightarrow\infty$.
Similarly, if $g<\lambda$, then $\rho_{i',t'}^{(k)}$ will converge to $0$.
% Equation \eqref{eq:rhorho} shows that if $g>\lambda$, then $\rho_{i',t'}^{(k)}$ will converge to $1$ as $k$ increases.
% Similarly, if $g<\lambda$, then $\rho_{i',t'}^{(k)}$ will converge to $0$.
\end{proof}

\begin{remark}[$\rho$ is a gradient based explanation]\label{remark:rho}
% The gradient $g$ in Equation~\eqref{eq:g} at $\rho_{i',t'}$ can be interpreted as the importance of the causal relationship $\mathbf{h}^{in}_{i',t'}\rightarrow\mathbf{h}^{out}_{i,t}$.
From the proof of Theorem \ref{thm:convergence}, it can be noted that,
if $g>\lambda$, then $\rho_{i',t'}\rightarrow1$; 
if $g<\lambda$, then $\rho_{i',t'}\rightarrow0$.
This phenomenon reflects that $\rho_{i',t'}$ serves as a binary indicator showing whether the gradient at $\rho_{i',t'}$ is greater or less than the threshold $\lambda$.
% If $\rho_{i',t'}=0$, then $\beta_{i',t'}\cdot\alpha_{i',t'}=0$, showing that $\mathbf{h}^{in}_{i',t'}$ does not Granger cause $\mathbf{h}^{out}_{i,t}$.
\end{remark}

\paragraph{Complexity Analysis}
The complexities of the input project, the skip project, and the transformer layer are $O(NT)$, $O(NT)$, and $O(NT^2)$.
The complexities of SCA  and PBD are $O(ET)$ and $O(NN_PT)$, where $E$ is the number of edges in the sensor network, and $N_P$ is the number of prompts.
The overall complexity is $O(\max(E,NT,NN_P)T)$.
In SCA, if we calculate attention weights and causal gates for each pair of the data points without using $\mathbf{A}$, then the complexity will be extremely high: $O(N^2T^2)\gg O(ET)$.

%% file: 04_experiments.tex
\section{Experiments}
% In this section, we provide experimental results to answer the following questions:
% \textbf{Q1.} How does \casper\ perform on the imputation task? 
% \textbf{Q2.} How does each component of \casper, i.e., PBD and SCA, contribute to the final performance?
% \textbf{Q3.} Can \casper\ discover the potential causal graphs?  

\begin{table*}[t]
    \footnotesize
    \centering
    \caption{Performance (MAE, MSE) of different methods.}\label{tab:mae}
    \setlength\tabcolsep{2pt}
    \begin{tabular}{c|cc|cc|cc|cc|cc|cc}
    \toprule
        & \multicolumn{4}{c|}{General Missing} & \multicolumn{4}{c|}{Point Missing} & \multicolumn{4}{c}{Block Missing}\\
        \midrule
        & \multicolumn{2}{c|}{AQI-36} & \multicolumn{2}{c|}{AQI} & \multicolumn{2}{c|}{METR-LA} & \multicolumn{2}{c|}{PEMS-BAY} & \multicolumn{2}{c|}{METR-LA} & \multicolumn{2}{c}{PEMS-BAY} \\
        \midrule
         & MAE & MSE & MAE & MSE & MAE & MSE & MAE & MSE & MAE & MSE & MAE & MSE \\
        \midrule
        Mean & 53.48$\pm$0.00  & 4578.08$\pm$00.00 & 39.60$\pm$0.00 & 3231.04$\pm$00.00 & 7.56$\pm$0.00 & 142.22$\pm$0.00 & 5.42$\pm$0.00 & 86.59$\pm$0.00 & 7.48$\pm$0.00 & 139.54$\pm$0.00 & 5.46$\pm$0.00 & 87.56$\pm$0.00\\
        KNN & 30.21$\pm$0.00 & 2892.31$\pm$00.00 & 34.10$\pm$0.00 & 3471.14$\pm$00.00 & 7.88$\pm$0.00 & 129.29$\pm$0.00 & 4.30$\pm$0.00 & 49.80$\pm$0.00 & 7.79$\pm$0.00 & 124.61$\pm$0.00 & 4.30$\pm$0.00 & 49.90$\pm$0.00\\
        MF & 30.54$\pm$0.26 & 2763.06$\pm$63.35 & 26.74$\pm$0.24 & 2021.44$\pm$27.98 & 5.56$\pm$0.03 & 113.46$\pm$1.08 & 3.29$\pm$0.01 & 51.39$\pm$0.64 & 5.46$\pm$0.02 & 109.61$\pm$0.78 & 3.28$\pm$0.01 & 50.14$\pm$0.13\\
        MICE & 30.37$\pm$0.09 & 2594.06$\pm$07.17 & 26.98$\pm$0.10 & 1930.92$\pm$10.08 & 4.42$\pm$0.07 & 55.07$\pm$1.46 & 3.09$\pm$0.02 & 31.43$\pm$0.41 & 4.22$\pm$0.05 & 51.07$\pm$1.25 & 2.94$\pm$0.02 & 28.28$\pm$0.37\\
        VAR & 15.64$\pm$0.08 & 833.46$\pm$13.85 & 22.95$\pm$0.30 & 1402.84$\pm$52.63 & 2.69$\pm$0.00 & 21.10$\pm$0.02 & 1.30$\pm$0.00 & 6.52$\pm$0.01 & 3.11$\pm$0.08 & 28.00$\pm$0.76 & 2.09$\pm$0.10 & 16.06$\pm$0.73\\
        \midrule
        rGAIN & 15.37$\pm$0.26 & 641.92$\pm$33.89 & 21.78$\pm$0.50 & 1274.93$\pm$60.28 & 2.83$\pm$0.01 & 20.03$\pm$0.09 & 1.88$\pm$0.02 & 10.37$\pm$0.20 & 2.90$\pm$0.01 & 21.67$\pm$0.15 & 2.18$\pm$0.01 & 13.96$\pm$0.20\\
        BRITS & 14.50$\pm$0.35 & 662.36$\pm$65.16 & 20.21$\pm$0.22 & 1157.89$\pm$25.66 & 2.34$\pm$0.00 & 16.46$\pm$0.05 & 1.47$\pm$0.00 & 7.94$\pm$0.03 & 2.34$\pm$0.01 & 17.00$\pm$0.14 & 1.70$\pm$0.01 & 10.50$\pm$0.07\\
        % SAITS & 18.16$\pm$0.42 & 21.33$\pm$0.15 & &  & 2.26$\pm$0.00 & 1.40$\pm$0.03 & &  & 2.30$\pm$0.01 & 1.56$\pm$0.01 & & \\
        \midrule
        ST-Transformer & 11.98$\pm$0.53 & 557.22$\pm$46.52 & 18.11$\pm$0.25 & 1135.46$\pm$89.27  & 2.16$\pm$0.00 & 13.66$\pm$0.03 & 0.74$\pm$0.00 & 1.96$\pm$0.03 & 3.54$\pm$0.00 & 52.22$\pm$0.99 & 1.70$\pm$0.02 & 20.37$\pm$0.43\\
        GRIN & 12.08$\pm$0.47 & 523.14$\pm$57.17 & 14.73$\pm$0.15 & 775.91$\pm$28.49 & 1.91$\pm$0.00 & 10.41$\pm$0.03 & 0.67$\pm$0.00 & 1.55$\pm$0.01 & 2.03$\pm$0.00 & 13.26$\pm$0.05 & 1.14$\pm$0.01 & 6.60$\pm$0.10\\
        SPIN & 11.77$\pm$0.54 & 455.53$\pm$12.27 & 13.92$\pm$0.15 & 773.60$\pm$26.64 & 1.90$\pm$0.01 & 18.47$\pm$0.31 & 0.70$\pm$0.01 & 1.91$\pm$0.01 & 1.98$\pm$0.01 & 18.47$\pm$0.31 & 1.06$\pm$0.02 & 7.42$\pm$0.16\\
        PoGeVon & 10.92$\pm$0.24 & 493.94$\pm$51.89 & 14.18$\pm$0.04 & 740.57$\pm$8.01 & 1.96$\pm$0.01 & 11.08$\pm$0.05 & 0.67$\pm$0.01 &  \textbf{1.51$\pm$0.03} & 1.95$\pm$0.01 & 13.08$\pm$0.08 & 1.54$\pm$0.02 & 17.18$\pm$0.48 \\
        \midrule
        \casper & \textbf{10.09$\pm$0.13} & \textbf{396.16$\pm$12.94} & \textbf{13.30$\pm$0.06} & \textbf{658.07$\pm$4.88} & \textbf{1.84$\pm$0.00} & \textbf{9.99$\pm$0.01} & \textbf{0.65$\pm$0.00} & 1.63$\pm$0.01 & \textbf{1.92$\pm$0.01} & \textbf{11.98$\pm$0.23} & \textbf{1.00$\pm$0.00} & \textbf{5.37$\pm$0.04}\\
        \bottomrule
    \end{tabular}
\end{table*}

\subsection{Experimental Setup}
In this subsection, we briefly explain the datasets, evaluation metrics, and baselines used for the experiments.

\paragraph{Datasets.}
Three public real-world benchmark datasets are used to evaluate \casper.
\textbf{AQI} \cite{zheng2015forecasting} is an hourly record of air pollutants from 437 air quality monitoring stations in China from May 2014 to April 2015.
We also use the popular \textbf{AQI-36} \cite{cao2018brits} which is a reduced version of the full AQI containing records from 36 sensor stations scattered around Beijing.
\textbf{METR-LA} \cite{li2018diffusion} contains traffic speed time series collected from 207 sensors on highways in Los Angeles for 4 months.
\textbf{PEMS-BAY} \cite{li2018diffusion} is a traffic speed time series collected from 325 sensors on highways in San Francisco Bay Area for 6 months.
The time series records in METR-LA and PEMS-BAY are collected every 5 minutes.
For AQI, METR-LA, and PEMS-BAY, the temporal window is set as $T=24$, and for AQI-36, the temporal window is 36.
% Following \cite{marisca2022learning}, the sensor network is built based on the geographical distances between sensor stations.
To be consistent with prior works \cite{marisca2022learning}, the adjacency matrices of sensor networks are built by applying a thresholded Gaussian kernel \cite{shuman2013emerging,li2018diffusion} over the geographical distances between sensor stations.

For AQI and AQI-36, we use the evaluation masks in \cite{yi2016st,marisca2022learning} which simulates the real missing data distribution in the datasets.
We refer to this setting as the general missing.
For METR-LA and PEMS-BAY, we consider both points missing and block missing settings as in \cite{marisca2022learning}.
In point missing, 25\% data points are masked out.
In block missing, 5\% spatial blocks and 0.15\% temporal blocks ranging from 1 hour to 4 hours are masked out.

\paragraph{Evaluation Metrics.}
We use the standard Mean Absolute Error (MAE) and Mean Squared Error (MSE) between the ground truth values and the imputed values as the evaluation metrics.

\paragraph{Baselines.}
We consider three groups of methods as our baselines.
(1) Traditional statistical methods: 
the mean value of the sequence (MEAN); neighbor mean (KNN);
matrix factorization (MF); 
multiple imputations using chained equations (MICE \cite{white2011multiple});
vector auto-regression (VAR);
% The results of baselines are copied from the corresponding papers.
(2) Early deep learning models:
rGAIN \cite{cini2021filling}: an adversarial method similar to \cite{miao2021generative,luo2018multivariate};
BRITS \cite{cao2018brits}: a bidirectional RNN imputation method;
(3) Recent deep learning models:
ST-Transformer \cite{marisca2022learning}: a spatiotemporal extension of the original Transformer \cite{vaswani2017attention};
GRIN \cite{cini2021filling}: a graph enhanced recurrent neural network;
SPIN \cite{marisca2022learning}: a spatiotemporal graph attention based imputation model;
PoGeVon \cite{wang2023networked}: a recent spatiotemporal imputation method which is based on the position-aware spatiotemporal graph variational auto-encoder.
Whenever possible, the results of baselines are copied from the corresponding paper.

\paragraph{Implementation Details.}
Most of the training configurations follow prior works \cite{marisca2022learning}.
We set the size of embeddings as 32.
The numbers of layers for the encoder for AQI-36 and other datasets are 2 and 4 respectively.
We use the Adam optimizer \cite{kingma2014adam} with a learning rate of 0.0008 and a cosine scheduler to train the model.
The maximum number of epochs is 300 and the patience of early stopping is 40 epochs.
Batch size is fixed as 8.
During training, $p\in[0.2, 0.5, 0.8]$ data points are randomly masked out for each batch, and the loss is calculated based on these masked points.

\subsection{Overall Performance}
In this subsection, we show the overall performance (MSE and MAE) of \casper\ for the imputation task to demonstrate the overall competence of \casper\ for imputation.

The overall performance of different methods is presented in Table \ref{tab:mae}.
The upper, middle, and lower groups of baselines are the traditional statistical methods, the RNN methods and the recent methods (Transformer based and graph based methods).
Generally speaking, the RNN methods perform better than the statistical methods, and the recent methods further outperform the RNN methods.
When imputing a data point, these methods exploit all the available information in the context, without identifying the causal relationships between the data point and the context.
However, it is inevitable that some confounders are included in the data, such as the non-causal shortcut edges.
Over-reliance on the confounders could lead to overfitting and make the model susceptible to noise.
The proposed \casper\ could effectively remove the impact of confounders.
As shown in Table \ref{tab:mae}, \casper\ achieves the lowest overall MAE and MSE scores and also has lower standard deviations, demonstrating the effectiveness of enforcing the model to discover causality during training.
% We present further investigation for \casper\ in the following subsections.

\begin{table}[h]
    \footnotesize
    \centering
    \caption{Ablation study on the AQI-36 dataset.}\label{tab:ablation}
    \begin{tabular}{l|cc}
        \toprule
        & MAE & MSE \\
        \midrule
        \casper & \textbf{10.09$\pm$0.13} & \textbf{396.16$\pm$12.94}\\
        \midrule
        w/o PBD (i.e., PBD$\rightarrow$MLP) & 10.36$\pm$0.11 & 445.03$\pm$29.03\\
        w/o SCA (i.e., w/o $\beta$) & 10.45$\pm$0.17 & 426.64$\pm$34.41\\
        w/o PBD, SCA & 10.84$\pm$0.20 & 472.24$\pm$42.77\\
        w/o PBD, SCA, $\mathbf{A}$ & 14.86$\pm$0.21 & 767.96$\pm$25.32\\
        \midrule
        Prompts $\rightarrow$ K-means Centers (max)  & 10.18$\pm$0.14 & 427.20$\pm$28.84\\ 
        Prompts $\rightarrow$ K-means Centers (avg)  & 10.23$\pm$0.15 & 421.43$\pm$19.40\\ 
        Prompts $\rightarrow$ Sampling (max) & 10.25$\pm$0.18 & 421.78$\pm$33.26\\
        Prompts $\rightarrow$ Sampling (avg) & 10.56$\pm$0.19 & 474.27$\pm$50.55\\
        \midrule
        unconstrained$\rightarrow$constrained causality & 10.83$\pm$0.19	& 477.85$\pm$45.62 \\
        w/o skip project & 10.35$\pm$0.21& 438.02$\pm$28.38 \\
        \bottomrule
    \end{tabular}

\end{table}

\subsection{Ablation Study}
In this subsection, we study the impact of different components in \casper, and the results are shown in Table \ref{tab:ablation}.

\paragraph{Effectiveness of PBD,SCA and $\mathbf{A}$.}
In the upper part of Table \ref{tab:ablation}, we investigate the impact of the Prompt-Based Decoder (PBD), the Spatiotemporal Causal Attention (SCA) and the sensor network $\mathbf{A}$.
First, a significant performance drop on MAE and MSE can be observed when we remove PBD (replace PBD with an MLP) and/or SCA (remove the causal gate $\beta$) from the full model \casper, 
indicating the effectiveness of the frontdoor adjustment and the causal gate for improving the overall performance. 
It is also worth noting that, on MSE, the standard deviations of the ablated versions of \casper\ (w/o PBD and/or SCA) are significantly higher than the full model \casper, demonstrating that enforcing the model to focus on the causality could improve its robustness.
% demonstrating the effectiveness of the frontdoor adjustment to eliminate the impact of confounders, and the effectiveness of the causal gate to filter out noisy non-causal relationships.
% This observation indicates that the unknown confounders and the non-causal correlations could mislead the model's decision making process.
% Enforcing the model to focus on causality could improve its performance.
Second, although $\mathbf{A}$ contains non-causal relationships among sensors, it still has critical contributions to the overall performance, demonstrating the necessity of considering $\mathbf{A}$ when imputing missing values.

\paragraph{Effectiveness of the Prompts.}
In the middle part of Table \ref{tab:ablation}, we demonstrate the effectiveness of using the learnable prompts to capture the global contextual information of the datasets by replacing the prompts with other approximations, including K-means cluster centers and randomly sampled data.
% Specifically, 
To obtain the other two approximations, 
we first pre-train an imputation model, i.e., \casper\ without SCA and PBD, and then extract the embedding of each training sample $\mathbf{G}$ by applying average/max pooling over the embeddings of all the points in $\mathbf{G}$.
For the cluster center approximation, we apply K-means over the embeddings to obtain 1,000 cluster centers.
For the sampling approximation, we randomly sample 1,000 embeddings for each training sample. 
Compared with the prompts, the cluster centers and randomly sampled embeddings not only perform worse but also require extra effort to obtain embeddings, which demonstrates the superiority of prompts.

\paragraph{Effectiveness of Other Components.}
In the lower part of Table \ref{tab:ablation}, we study the impact of time constraints and the skip project layer in the model.
First, if we enforce the time constraint, i.e., $t'\leq t$ in Definition \ref{def:granger_causality_emb}, for the imputation task, then the performance will significantly drop.
These results demonstrate that it is vital to take into consideration both the past and future points for the imputation task. 
Second, removing the skip project layer will have negative impacts on the overall performance, showing the power of the skip project, which aligns with the observation of the residue connections in the literature \cite{he2016deep,jing2021multiplex,yan2021dynamic}.

% \begin{figure*}[t]
% \centering
% \begin{subfigure}[b]{.24\textwidth}
% \includegraphics[width=\linewidth]{figs/x_700_copy.png}
%  \caption{Input}\label{fig:x_2}
% \end{subfigure}
% \begin{subfigure}[b]{.24\textwidth}
%   \includegraphics[width=\linewidth]{figs/final_mat_700_4_3.png}
%   \caption{\casper}\label{fig:final_mat2}
% \end{subfigure}
% \begin{subfigure}[b]{.24\textwidth}
%   \includegraphics[width=\linewidth]{figs/softmax_ablation_4_mat_700_4_3.png}
%   \caption{W/O SCA}\label{fig:softmax_4_2}
% \end{subfigure}
% \begin{subfigure}[b]{.24\textwidth}
%   \includegraphics[width=\linewidth]{figs/softmax_ablation_2_mat_700_4_3.png}
%   \caption{W/O SCA, PBD}\label{fig:softmax_2_2}
% \end{subfigure}
% \caption{The input matrix and attention maps of \casper\ and its ablated versions. The query point is marked by star.}
% \label{fig:att2}
% \end{figure*}

% \begin{figure}
%     \centering
%     \includegraphics[width=.35\textwidth]{figs/map_all.jpg}
%     \caption{The discovered causal relationships among sensors. The red and blue correspond to Figure \ref{fig:final_mat1} and Figure \ref{fig:final_mat2}.}
%     \label{fig:map}
% \end{figure}

\begin{figure*}[t]
\centering
\begin{subfigure}[b]{.24\textwidth}
\includegraphics[width=\linewidth]{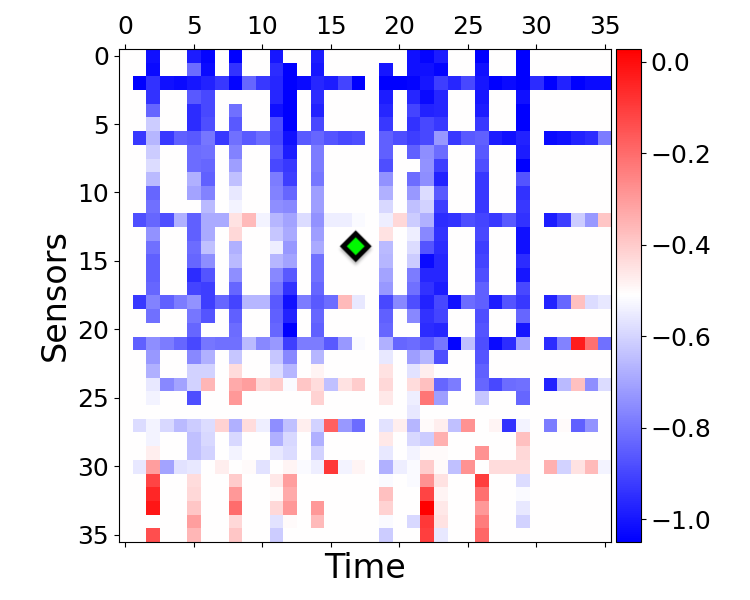}
 \caption{Input}\label{fig:x_1}
\end{subfigure}
\begin{subfigure}[b]{.24\textwidth}
  \includegraphics[width=\linewidth]{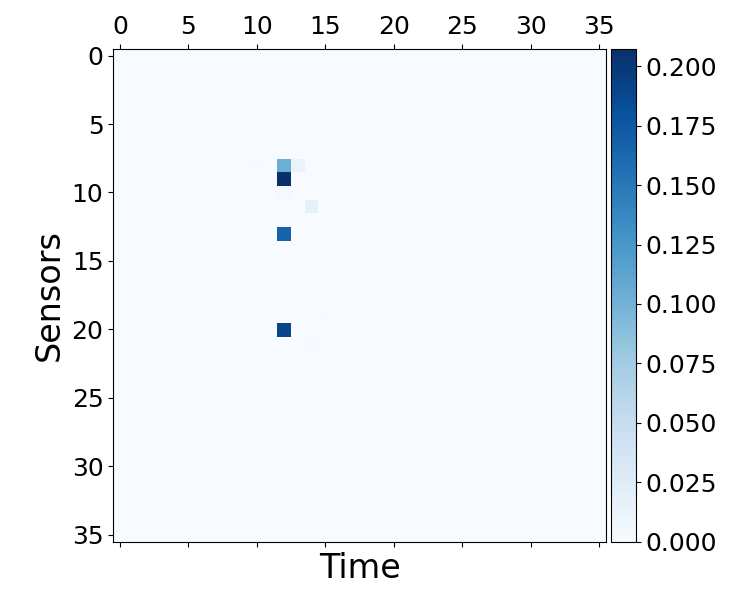}
  \caption{\casper}\label{fig:final_mat1}
\end{subfigure}
\begin{subfigure}[b]{.24\textwidth}
  \includegraphics[width=\linewidth]{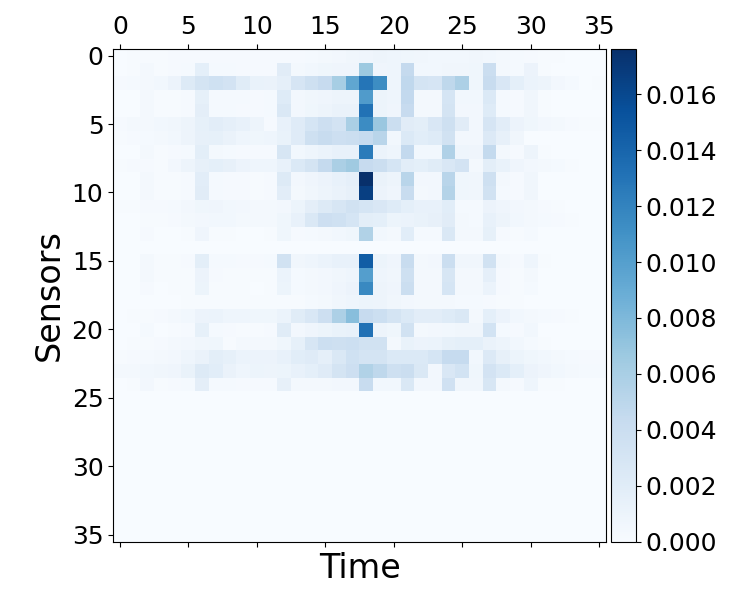}
  \caption{\casper\ w/o SCA}\label{fig:softmax_4_1}
\end{subfigure}
\begin{subfigure}[b]{.24\textwidth}
  \includegraphics[width=\linewidth]{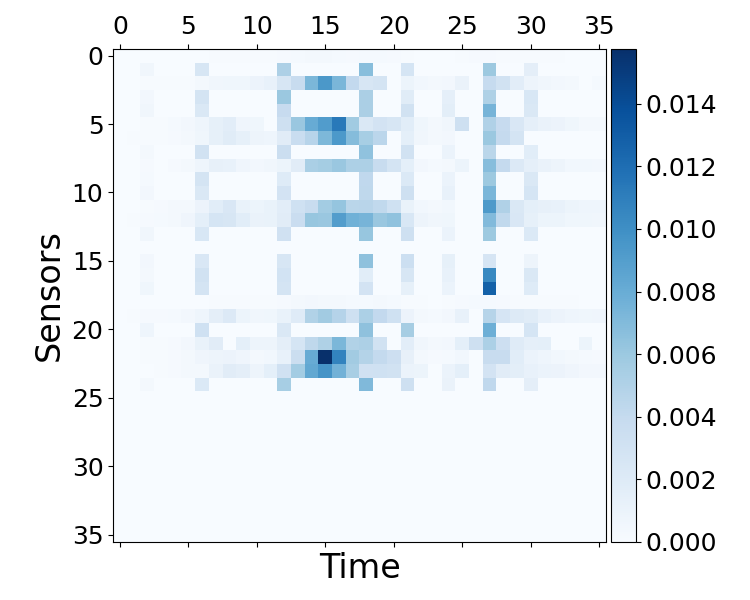}
  \caption{\casper\ w/o SCA, PBD}\label{fig:softmax_2_1}
\end{subfigure}
\begin{subfigure}[b]{.24\textwidth}
\includegraphics[width=\linewidth]{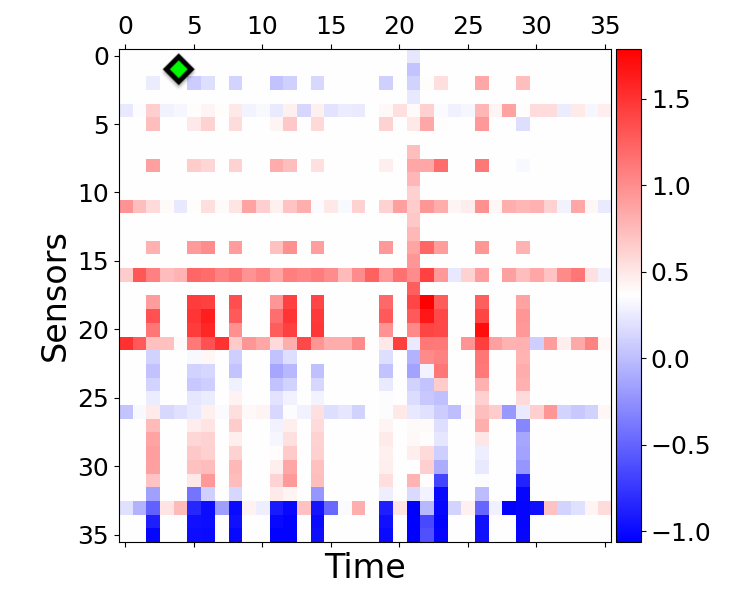}
 \caption{Input}\label{fig:x_2}
\end{subfigure}
\begin{subfigure}[b]{.24\textwidth}
  \includegraphics[width=\linewidth]{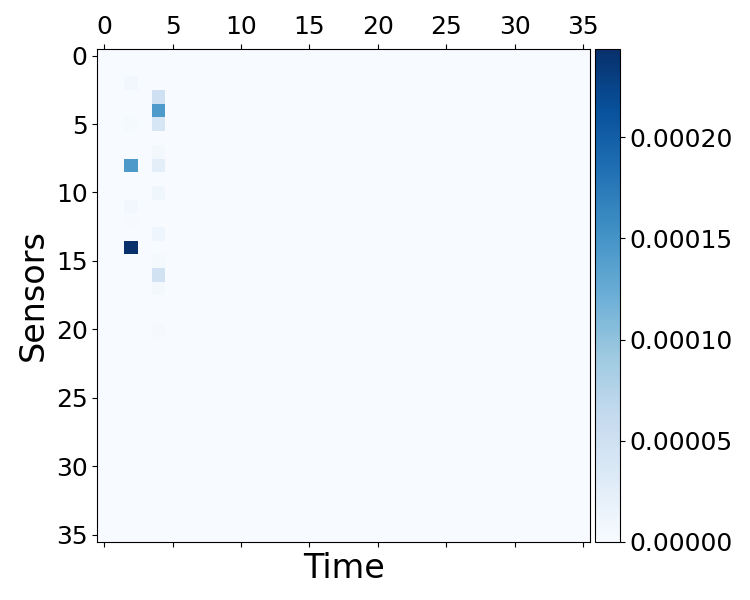}
  \caption{\casper}\label{fig:final_mat2}
\end{subfigure}
\begin{subfigure}[b]{.24\textwidth}
  \includegraphics[width=\linewidth]{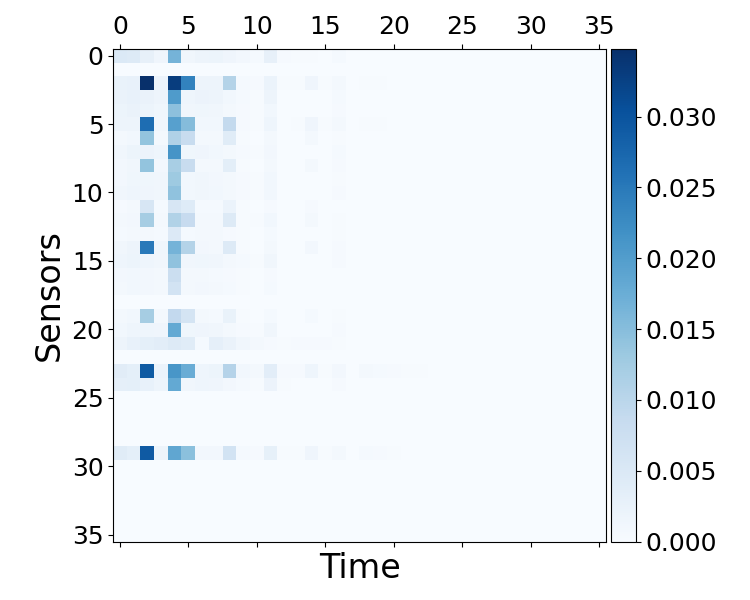}
  \caption{\casper\ w/o SCA}\label{fig:softmax_4_2}
\end{subfigure}
\begin{subfigure}[b]{.24\textwidth}
  \includegraphics[width=\linewidth]{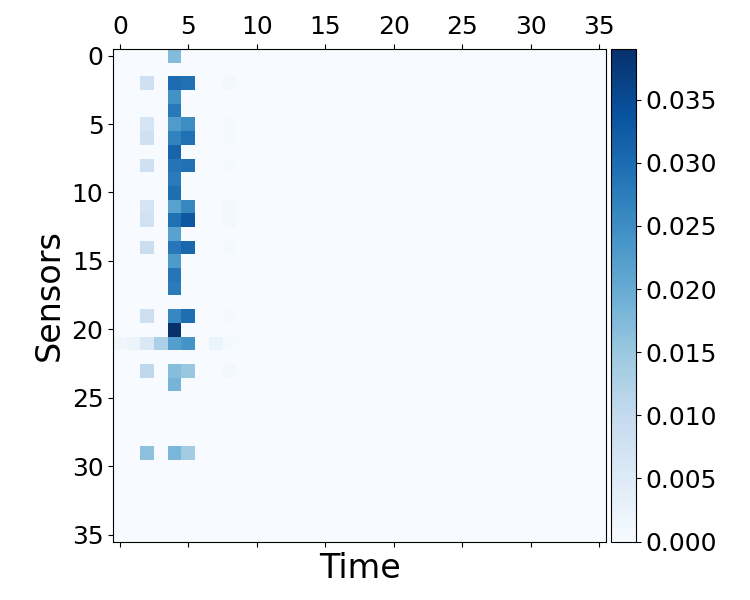}
  \caption{\casper\ w/o SCA, PBD}\label{fig:softmax_2_2}
\end{subfigure}
\caption{Input matrices and the associated attention maps of \casper\ and its ablated versions. Diamonds are the query points.}
\label{fig:att1}
\end{figure*}

\begin{figure}[t]
\centering
\begin{subfigure}[b]{.235\textwidth}
\includegraphics[width=\linewidth]{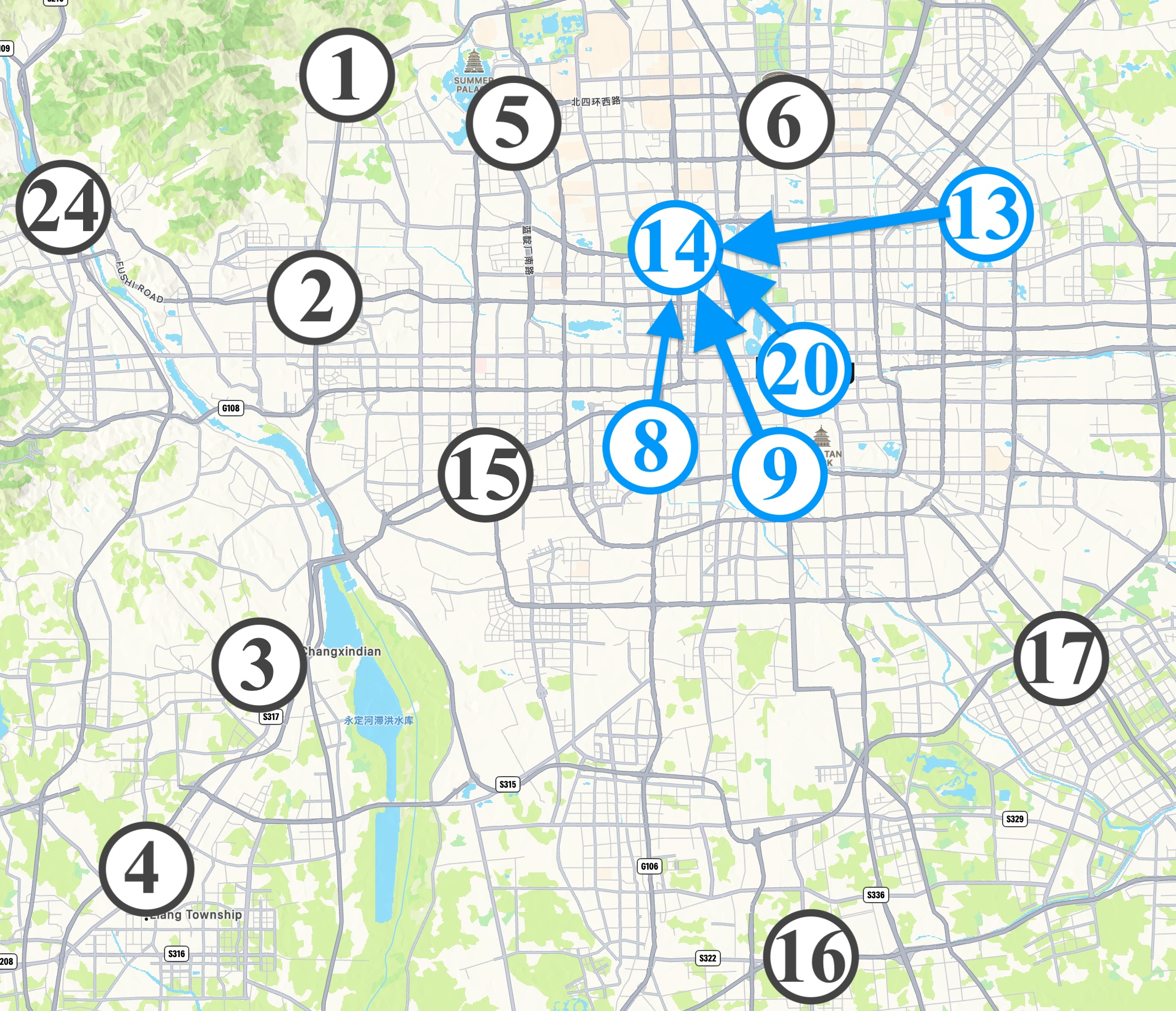}
 \caption{Causal relations in Figure \ref{fig:final_mat1}.}\label{fig:causal1}
\end{subfigure}
\begin{subfigure}[b]{.235\textwidth}
  \includegraphics[width=\linewidth]{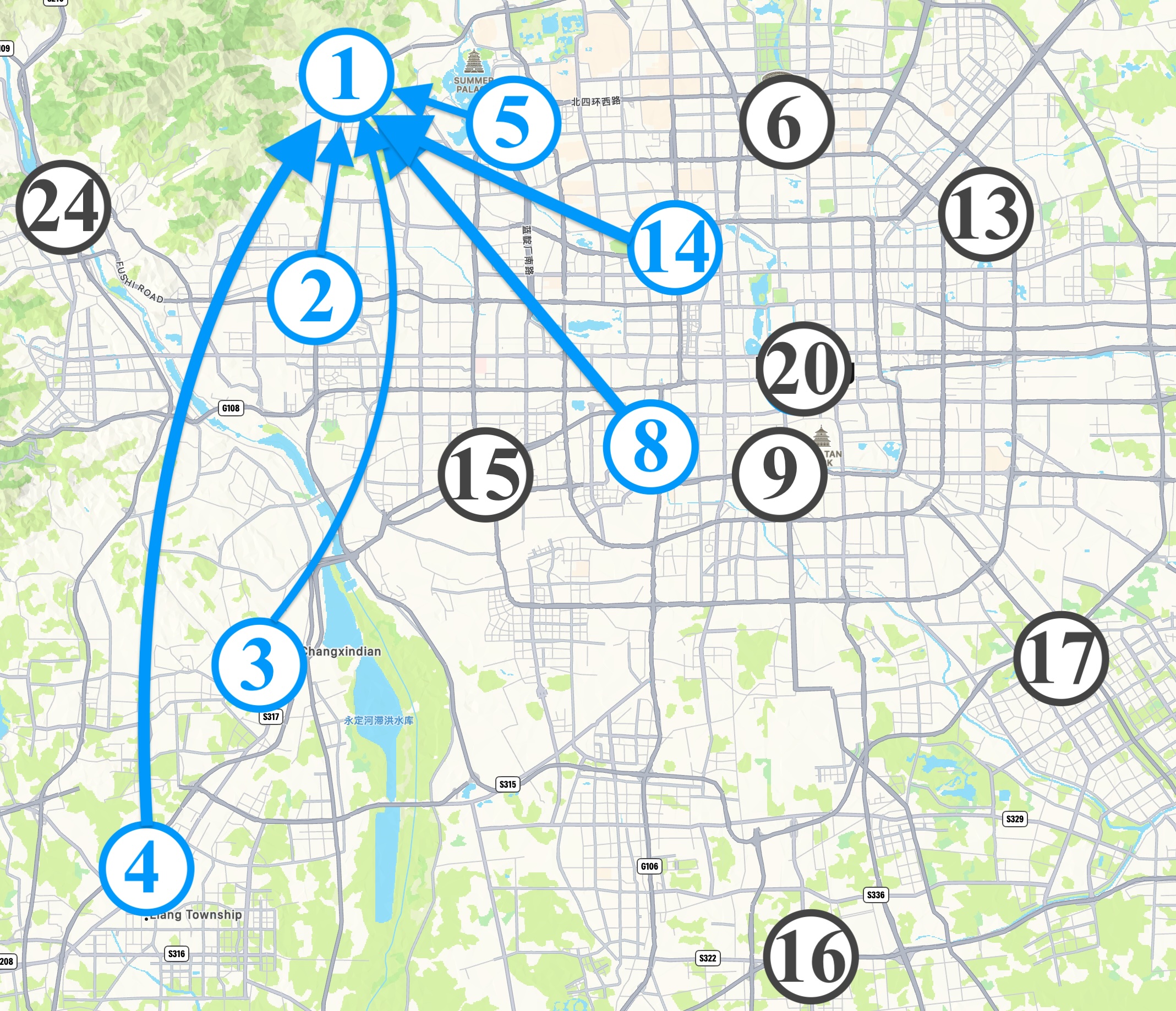}
  \caption{Causal relations in Figure \ref{fig:final_mat2}.}\label{fig:causal2}
\end{subfigure}
\caption{Discovered causal relationships.}
\label{fig:maps}
\end{figure}

\subsection{Visualization}\label{sec:exp_vis}
In this subsection, we provide visualization to further analyze \casper's ability of causality discovery.

\paragraph{Attention Maps.}
In Figure \ref{fig:att1}, we visualize two input time series from the test set of AQI-36, and their corresponding attention maps of the last encoder layer from different models, i.e., \casper, \casper\ w/o SCA (no causal gate) and \casper\ w/o SCA, PBD (no causal gate, PBD$\rightarrow$MLP).
Figure \ref{fig:x_1} and \ref{fig:x_2} are the inputs and the diamonds are the target query points.
Other figures show the attention scores between the query point and all other points in the context.
Figure \ref{fig:final_mat1} and \ref{fig:final_mat2} are the causal aware attention scores $\beta\cdot\alpha/Z$ in Equation \eqref{eq:sca}.
Figure \ref{fig:softmax_4_1}-\ref{fig:softmax_2_1} and Figure \ref{fig:softmax_4_2}-\ref{fig:softmax_2_2} are the attention scores $\alpha$.
First, by comparing the last two figures in each row, we can observe that the attention maps learned by \casper\ w/o SCA are sparser than \casper\ w/o SCA, PBD. 
This observation demonstrates that the frontdoor adjustment can effectively remove the noise and confounders by forcing the model to focus on only a small set of important points.
Second, by comparing \casper\ and \casper\ w/o SCA, we can observe that SCA further improves the sparsity of attention by focusing on only a few points.
Since $l_1$ norm is placed over the causal probabilities $\rho$ during training, therefore, points with non-zero attention weights are critical for the query point, which cannot be removed.
According to the Granger causality (Definition \ref{def:granger_causality_emb}), these non-zero points are the causes for the query.

\paragraph{Discovered Causal Relationships.}
We draw the most salient causal relationships in Figure \ref{fig:final_mat1} and \ref{fig:final_mat2} on the map of Beijing in Figure \ref{fig:causal1} and \ref{fig:causal2}.
Each number corresponds to a sensor and the arrow width corresponds to the attention weights.
% For each row in Figure \ref{fig:final_mat1} and \ref{fig:final_mat2}, if its value is greater than a certain threshold, then we regard there's a causal re
% Specifically, for a given causality aware attention matrix $\mathbf{C}\in\mathbb{R}^{N\times T}$ for $x_{i,t}$, which is obtained from Equation \eqref{eq:sca}, we sum the second dimention of $\mathbf{C}$ and obtain a causal weight vector of size $N$: $\mathbf{c}[i']=\sum_{t'=1}^T\mathbf{C}[i',t']$.
% If $\mathbf{c}[i']>\tau$, where $\tau$ is a threshold, then there is a causal relationship between $x_{i'}$ and $x_{i}$.
% The AQI-36 dataset is collected from 36 air quality monitoring stations in Beijing, which record PM2.5 levels.
% We draw the sensors on the map of Beijing according to their latitudes and longitudes.
% The widths of the arrows correspond to the weights in Figure \ref{fig:final_mat1} and \ref{fig:final_mat2}.
It is evident that for the query sensors 14 and 1 in the two figures, not every nearby sensor has a causal relationship with them.
In Figure \ref{fig:causal1}, although the sensors 6, 5, 15 are very close to the query 14, \casper\ discovers that there is no causal relationship among them.
We conjecture that this is because the wind blew generally westward and northward in Beijing for Figure \ref{fig:final_mat1} and \ref{fig:final_mat2} \cite{zhang2012impact}.
In Figure \ref{fig:causal2}, the neighboring sensors 24,  15, 6, 20, 9, 13 are not regarded as the causes for the query sensor 1.
According to the spatial relationships, the information of the sensors 15, 6, 20, 9, 13 might be included in the causal sensors such as 2, 5 14, and 8, since the causal sensors are in between with the query and these non-causal sensors.
However, there is no other sensor between 1 and 24.
By taking a closer look at the upper-left corner of the map, we find that sensors 1 and 24 are separated by the Fragrant Hills. 
Therefore, the air quality at sensor 1 might not be directly influenced by sensor 24 for the period of the input data.
The two examples in Figure \ref{fig:maps} show that the proposed \casper\ could effectively discover the causal relationships among sensors, which provides better insights for further data analysis.
Additionally, the two examples also show that the simple distance-based sensor network construction is biased, which contains many non-causal correlations since it ignores other factors in the real world, such as wind direction and terrain.

\subsection{Causal Graph Discovery on the Quasi-Realistic Dataset}
In the real world, the ground truth causal graphs are usually unavailable, and thus it is difficult to quantitatively evaluate the ability of causal graph discovery. 
Therefore, we follow \cite{cheng2022cuts} and evaluate our \casper\ on a quasi-realistic data called DREAM-3 \cite{prill2010towards}, which is a gene expression data regulation dataset.

DREAM-3 contains $N=100$ gene expression levels and the length of the expressions is $T=21$.
The goal is to discover the causal relationship among the 100 gene expression levels.
For DREAM-3, we train \casper\ via the imputation task, and use similar training configurations as  AQI-36.
Following \cite{cheng2022cuts}, we use AUC between the ground-truth graph and the discovered graph as the evaluation metric.
Specifically, for \casper\ we obtain a causal weight matrix $\mathbf{C}_{i,t}\in\mathbb{R}^{N\times T}$ for each data point $x_{i,t}$, where $\mathbf{C}_{i,t}[i',t']=\beta_{i',t'}$ in Equation \ref{eq:gumbel}.
By concatenating all the $\mathbf{C}_{i,t}$ together, where $i\in\{1,\cdots,N\}$ and $t\in\{1,\cdots,T\}$, we have a tensor $\mathcal{C}\in\mathbb{R}^{N\times T\times N \times T}$, where the first two dimensions correspond to the data point $x_{i, t}$, and the last two dimensions correspond to the causal weight $\mathbf{C}_{i,t}$ for $x_{i,t}$.
We obtain the final causal weight matrix by max pooling over the two dimensions corresponding to the time, and normalizing by $T^2$:
$\mathbf{A}=\text{max-pool}(\mathcal{C},\text{dim=}2,4)/T^2\in\mathbb{R}^{N\times N}$.

The results of \casper\ and several SOTA time series causal discovery baselines are presented in Table \ref{tab:dream3}, where the results of baselines are copied from \cite{cheng2022cuts}.
Table \ref{tab:dream3} shows that \casper\ achieves the highest AUC score.
This experiment quantitatively demonstrates that \casper\ has a strong ability of discovering causal relationships.

\begin{table}[t!]
    \footnotesize
    \setlength\tabcolsep{2pt}
    \centering
    \caption{Causal Graph Discovery on DREAM-3.}\label{tab:dream3}
    \begin{tabular}{l|ccccccc}
        \toprule
        Models & PCMCI\cite{runge2019detecting} & NGC\cite{tank2021neural} & eSRU\cite{khanna2019economy} & LCCM\cite{de2020latent} & NGM\cite{bellot2021neural} & CUTS\cite{cheng2022cuts} & \casper\\
        \midrule
        AUC & 0.5517 & 0.5579 & 0.5587 & 0.5046 & 0.5477 & 0.5915 & \textbf{0.6325}\\
        \bottomrule
    \end{tabular}
\end{table}

\subsection{More Results}
In this subsection, we provide more experimental results of \casper, including convergence of $\rho$, and sensitivity analysis. 

\paragraph{Convergence of $\rho$.}
We set $0.1$ and $0.9$ as the thresholds to round $\rho$. 
Specifically, if $\rho\leq0.1$ then we regard $\rho$ has converged to $0$; similarly, if $\rho\geq0.9$, then we regard $\rho$ has converged to $1$.
The statistical results of the AQI-36 dataset show that 98\% $\rho$ converges to 0 or 1, which corroborates Theorem \ref{thm:convergence}.

\paragraph{Sensitivity Analysis.}
We present the results of sensitivity experiments of $\lambda$ and the number of prompts $N_P$ in Figure \ref{fig:lambda}-\ref{fig:prompts}.
For $\lambda$, the lowest MAE can be obtained when $\lambda$ is around 0.001.
For $N_P$, the best results can be obtained when $N_P\in[200,1400]$.
% The best results can be obtained when $\lambda\sim 0.001$ and $N_p\in[200,1400]$.

\textbf{\begin{figure}[t]
\centering
\begin{subfigure}[b]{.236\textwidth}
\includegraphics[width=\linewidth]{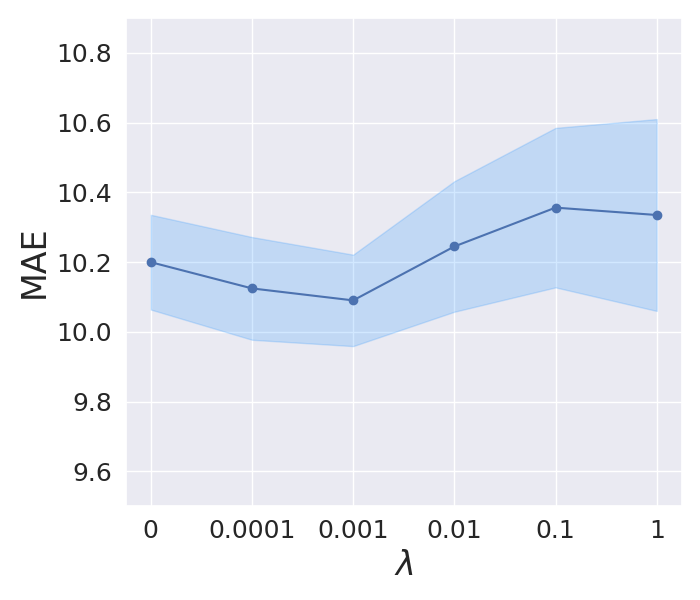}
 \caption{The weight of $l_1$ norm $\lambda$.}\label{fig:lambda}
\end{subfigure}
\begin{subfigure}[b]{.236\textwidth}
  \includegraphics[width=\linewidth]{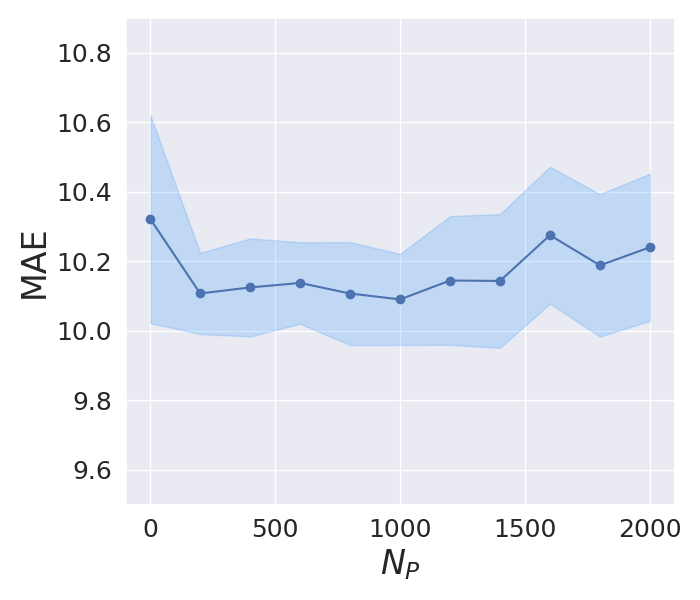}
  \caption{The number of prompts $N_P$.}\label{fig:prompts}
\end{subfigure}
\caption{Sensitivity experiments.}
\label{fig:sensitivity}
\end{figure}}

%% file: 05_related_work.tex
\section{Related Work}
In this section, we briefly review the most relevant works to ours, including spatiotemporal time series methods as well as causal inference methods.

\subsection{Spatiotemporal Time Series Imputation}
Spatiotemporal time series imputation is one of the fundamental tasks for time series analysis \cite{esling2012time,wang2024deep,zheng2024heterogeneous}.
Traditional machine learning approaches are based on statistical analysis, such as linear autoregression \cite{yi2016st,durbin2012time} and matrix factorization \cite{yu2016temporal}.
At present, deep learning methods have become the mainstream.
Most existing deep learning methods are based on Recurrent Neural Networks (RNN).
GRU-D \cite{che2018recurrent} is one of the first RNN-based imputation models.
BRITS \cite{cao2018brits} leverage bi-directional RNN to impute missing data.
GAIN \cite{luo2018multivariate} and E2GAN \cite{luo2019e2gan} further apply Generative Adversarial Network (GAN) \cite{goodfellow2014generative} to enhance the performance.
mTAND \cite{shukla2020multi} adds attention mechanism to RNN.
These methods suffer from error propagation and accumulation brought by the auto-regression nature of RNN.
To address this issue, non-autoregressive methods are proposed such as NAOMI \cite{liu2019naomi}, NRTSI \cite{shan2023nrtsi} and the recent Transformer \cite{vaswani2017attention} based methods \cite{yildiz2022multivariate}.
There are also some other types of methods, such as Ordinary Differential Equations (ODEs) methods \cite{rubanova2019latent,rubanova2019latent}, state space models \cite{alcaraz2022diffusion} and diffusion models \cite{tashiro2021csdi}
The above methods mainly focus on the temporal patterns of time series, yet largely ignore the spatial relationships, e.g., distance, among time series.
To capture spatial relationships, graph neural networks \cite{kipf2016semi,velivckovic2017graph} are extended to the spatiotemporal setting.
LG-ODE \cite{huang2020learning} combines graph neural networks with ODE methods \cite{rubanova2019latent}.
RETIME \cite{jing2022retrieval} introduces a retrieval-based time series model, which leverages retrieved time series as an augmentation for the target time series.
NET$^3$ \cite{jing2021network} introduces a tensor graph neural network to model the high-order relationships among time series.
GRIN \cite{cini2021filling} introduces a bidirectional message passing RNN with a spatial decoder.
SPIN \cite{marisca2022learning} presents a sparse spatiotemporal graph neural network for spatiotemporal time series imputation.
Recently, PoGeVon \cite{wang2023networked} proposes a position-aware graph neural network based variational auto-encoder to impute both time series and edges.
However, these methods try to exploit all the available related information for the target missing point, without distinguishing the causal and non-causal relationships, which might have the overfitting problem and make the model vulnerable to noise.
Our proposed \casper\ could distinguish causal and non-causal relationships.

\subsection{Causal Inference}
Causality theory \cite{pearl2018book} provides theoretical guidance to design causality-aware models.
It has been widely explored in the computer vision domain to discover causal relationships \cite{bengio2019meta}, generate counter-factual samples \cite{abbasnejad2020counterfactual,yue2021counterfactual,kocaoglu2018causalgan} and reduce bias \cite{hu2021distilling,yang2021causal,qi2020two}.
In the graph mining domain, 
CGI \cite{feng2021should} studies how to select trustworthy neighbors during inference;
CLEAR \cite{ma2022clear} explores how to generate counterfactual explanations for graph-level prediction models based on Independent Component Analysis (ICA) \cite{khemakhem2020variational};
HyperSCI \cite{ma2022learning} explores the Individual Treatment Effect (ITE) on hyper-graphs.
NEAT \cite{ma2023look} investigates the impact of 
% Methicillin-Resistant Staphylococcus Aureus (MRSA) 
MRSA infection
via the Neyman-Rubin potential outcome framework \cite{rubin2005causal}.
CAL \cite{sui2022causal} introduces a causal attention learning framework for graph neural networks based on the backdoor adjustment \cite{pearl2018book}.
There are two differences between \casper\ and the above graph methods: 
(1) our setting is the \emph{dynamic} spatiotemporal setting but their setting is \emph{static} graph;
(2) \casper\ is based on the frontdoor adjustment and Granger causality, which is fundamentally different from their theoretical basis of causality.
In the time series domain, the
Granger causality \cite{granger1969investigating} is widely used for analyzing the causality between time series in the forecasting setting.
GrID-Net \cite{wu2021granger} leverages the Granger causality to infer regulatory locus–gene links.
cLSTM \cite{tank2021neural} and economy-SRU \cite{khanna2019economy} integrates the Granger causality with LSTM \cite{hochreiter1997long} and SRU \cite{oliva2017statistical}.
However, these methods require the input time series to be fully observed.
CUTS \cite{cheng2022cuts} is a recently proposed two-stage model, which first imputes missing data and then discovers causality between time series.
There are three major differences between \casper\ and CUTS.
(1) CUTS does not consider confounders, while \casper\ removes confounders via the frontdoor adjustment.
(2) CUTS has to re-train the causal model for each input segment, but \casper\ does not have such a requirement. 
(3) CUTS is a two-stage model, while \casper\ is a one-stage end-to-end model.

%% file: 06_conclusion.tex
\section{Conclusion}
In this paper, we review the spatiotemporal time series imputation task via the Structure Causal Model (SCM), which shows the causal relationships among the input, output, embeddings, and confounders.
The confounders could open shortcut backdoor paths between the input and output, which could mislead the model to learn the non-causal relationships.
We use the frontdoor adjustment to block the backdoor paths.
Based on the results of the frontdoor adjustment, we propose a novel Causality-Aware Spatiotemporal Graph Neural Network (\casper), which is comprised of a Prompt Based Decoder (PBD) and an encoder equipped with Spatiotemporal Causal Attention (SCA).
The proposed PBD could reduce the impact of the confounders at a high level.
For SCA, we first extend the definition of Granger causality for time series to embeddings.
Then we introduce the architecture of SCA based on the definition, which could discover the sparse causal relationships among embeddings.
Theoretical analysis shows that SCA decides causal and non-casual relationships based on the values of gradients.
Experimental results on three real-world benchmark datasets show that \casper\ could outperform the baseline methods for the imputation task.
Further analysis shows that \casper\ could effectively discover the sparse causal relationships.